\documentclass[11pt]{article}

\input{preamble.tex}
\usepackage[mathscr]{euscript}

\renewcommand{\(}{\left(}
\renewcommand{\)}{\right)}
\renewcommand{\[}{\left[}
\renewcommand{\]}{\right]}
\newcommand{\la}{\left|}
\newcommand{\ra}{\right|}
\newcommand{\lv}{\left\lVert}
\newcommand{\rv}{\right\rVert}
\newcommand{\lbr}{\left\{}
\newcommand{\rbr}{\right\}}

\newcommand{\cardp}{\abs}

\newcommand{\R}{\mathbb{R}}
\newcommand{\N}{\mathbb{N}}
\newcommand{\eps}{\varepsilon}
\newcommand{\diffp}{\eps}

\newcommand{\E}[2][]{\mathbb{E}_{\ifx &#1& \else #1 \fi}\left[#2\right]}
\newcommand{\1}[1]{\mathbf{1}\lbr #1\rbr}
\renewcommand{\P}[1]{\mathbb{P}\(#1\)}

\newcommand{\Lap}{\laplace}

\newcommand{\tr}{\text{tr}}

\newcommand{\cX}{\mathcal{X}}

\newcommand{\cE}{\mathcal{E}}

\newcommand{\close}[2]{\stackrel{d}{=}_{#1, #2}}
\newcommand{\closeed}{\close{\eps}{\delta}}

\newcommand{\opnorm}[1]{\norm{#1}_{\textup{op}}}
\newcommand{\lfro}[1]{\norm{#1}_{\textup{Fr}}}

\newcommand{\cas}{\stackrel{a.s.}{\rightarrow}}

\newcommand{\normal}{\mathsf{N}}
\newcommand{\laplace}{\mathsf{Lap}}
\newcommand{\uniform}{\mathsf{Uni}}

\newcommand{\cov}{\mathsf{Cov}}

\newcommand{\wb}[1]{\overline{#1}}

\newcommand{\Ep}{\mathbb{E}}

\renewcommand{\P}{\mathbb{P}}

\newcommand{\brc}[1]{\left\{{#1}\right\}} 
\newcommand{\norm}[1]{\left\|{#1}\right\|} 
\newcommand{\norms}[1]{\|{#1}\|} 
\newcommand{\normbig}[1]{\big\|{#1}\big\|} 
\newcommand{\nucnorm}[1]{\norm{#1}_*} 
\newcommand{\nucnorms}[1]{\norms{#1}_*} 
\newcommand{\nucnormbig}[1]{\normbig{#1}_*} 
\newcommand{\ltwo}[1]{\norm{#1}_2}  
\newcommand{\ltwobig}[1]{\normbig{#1}_2}  
\newcommand{\ltwos}[1]{\norms{#1}_2}  
\newcommand{\lone}[1]{\norm{#1}_1}
\newcommand{\linf}[1]{\norm{#1}_\infty}

\newcommand{\abs}[1]{\left|{#1}\right|} 
\newcommand{\est}[1]{\widehat{#1}}

\newcommand{\half}{\frac{1}{2}}

\newcommand{\mc}[1]{\mathcal{#1}}

\newcommand{\indic}[1]{1\!\left\{#1\right\}} 
\renewcommand{\>}{\rangle}

\newcommand{\defeq}{:=}

\definecolor{innerboxcolor}{rgb}{.9,.95,1}
\definecolor{outerlinecolor}{rgb}{.6,0,.2}
\definecolor{shcolor}{RGB}{27, 87, 14}
\definecolor{rckcolor}{RGB}{0,0,255}

\newcommand{\simiid}{\stackrel{\textup{iid}}{\sim}}

\newcommand{\txt}[1]{\textup{#1}}

\newcommand{\rset}[1]{{#1}}
\newcommand{\set}[1]{{#1}}

\newcommand{\wt}[1]{\widetilde{#1}}

\newcommand{\dpsd}{d_{\textup{psd}}}
\newcommand{\colspace}{\mathsf{Col}}
\newcommand{\genpsd}{A}
\newcommand{\Partitions}{\mathscr{P}}
\newcommand{\partitionset}{S}  
\newcommand{\partition}{\mc{S}} 

\newcommand{\hinge}[1]{\left({#1}\right)_+}  

\newcommand{\clcode}{\textup{cov}}
\newcommand{\clipr}{\hyperref[algorithm:clip-and-repeat]{\mathtt{COVSAFE}}}
\newcommand{\clnznormf}{z}  
\newcommand{\clnzthrf}{w}  
\newcommand{\clnznormr}{Z}  
\newcommand{\clnzthrr}{W}  

\newcommand{\clparams}{B,m}
\newcommand{\clx}{\tilde{x}}

\newcommand{\mscode}{\textup{mean}}
\newcommand{\meansafe}{\hyperref[algorithm:meansafe]{\mathtt{MEANSAFE}}}
\newcommand{\msparams}{B, b, k}
\newcommand{\msnzgf}{\partition}
\newcommand{\msnztopif}{z}
\newcommand{\msnztopiif}{z'}
\newcommand{\msnzthrf}{w}
\newcommand{\msnznf}{z^{\mathsf{N}}}
\newcommand{\msnzgr}{\partition}
\newcommand{\msnztopir}{Z}
\newcommand{\msnztopiir}{Z'}
\newcommand{\msnzthrr}{W}
\newcommand{\msnznr}{Z^{\mathsf{N}}}

\newcommand{\msmean}{\est{\mu}}
\newcommand{\msout}{\wt{\mu}}
\newcommand{\subtop}{_\textup{top}}

\newcommand{\adamean}{\hyperref[algorithm:adamean]{\mathtt{ADAMEAN}}}

\newcommand{\privmean}{\hyperref[algorithm:privmean]{\mathtt{PRIVMEAN}}}

\newcommand{\Wcovscale}{\sigma_{W_\clcode}}
\newcommand{\Wmeanscale}{\sigma_{W_\mscode}}
\newcommand{\mstopkscale}{\sigma\subtop}
\newcommand{\msnscale}{\sigma_{\mathsf{N}}}

\newcommand{\pmclnznormr}{Z_\clcode}
\newcommand{\pmclnzthrr}{W_\clcode}

\newcommand{\pmmsnztopiir}{Z\subtop'}
\newcommand{\pmmsnzthrr}{\msnzthrr_\mscode}

\newcommand{\clout}{\est{\Sigma}}

\newcommand{\pmout}{\widetilde{\mu}}

\newcommand{\topkdp}{\hyperref[algorithm:topk-dp]{\mathtt{TOPk}}}
\newcommand{\diam}{\txt{diam}}

\newcommand{\dsym}{d_{\txt{sym}}}

\newcommand{\samplemean}{\wb{X}_n}
\newcommand{\samplecov}{\wb{\Sigma}_n}
\newcommand{\whitenedcov}{\wb{\Sigma}_Z}
\newcommand{\assdiam}{M}
\newcommand{\assfail}{\beta}
\newcommand{\Es}{\cE_{\textup{samp}}}

\newcommand{\tstar}{t^\star}
\newcommand{\tstop}{t_{\textup{stop}}}

\begin{document}

\Crefname{algorithm}{Algorithm}{Algorithm}

\begin{center}
  \Large{A Fast Algorithm for Adaptive Private Mean Estimation} \\
  \vspace{.5cm}

  \large{John Duchi$^{1,2}$ ~~~~ Saminul Haque$^3$ ~~~~ Rohith Kuditipudi$^3$} \\
  \vspace{.25cm}
  \large{Departments of $^1$Statistics, $^2$Electrical Engineering,
    and $^3$Computer Science \\
    Stanford University}
  \\
  \vspace{.2cm}
  \large{January 2023}
\end{center}

\begin{abstract}
  We design an $(\diffp, \delta)$-differentially private algorithm
  to estimate the mean of a $d$-variate distribution, with unknown covariance
  $\Sigma$, that is adaptive to $\Sigma$. To within polylogarithmic factors,
  the estimator achieves optimal rates of convergence with
  respect to the induced Mahalanobis norm $\norm{\cdot}_\Sigma$,
  takes time $\wt{O}(n d^2)$ to compute, has near linear
  sample complexity for sub-Gaussian distributions, allows
  $\Sigma$ to be degenerate or low rank, and adaptively
  extends beyond sub-Gaussianity. Prior to this work, other methods
  required exponential computation time or the superlinear
  scaling $n = \Omega(d^{3/2})$ to achieve non-trivial error
  with respect to the norm $\norm{\cdot}_\Sigma$.
\end{abstract}


\section{Introduction}
\label{sec:introduction}

We cannot consider the theory of differential privacy complete until we
have---at least---a sample and computationally efficient estimator of the
mean.
To within logarithmic factors in the dimension $d$ and sample size $n$,
we achieve both.

To make this a bit more precise, let
$P$ be a distribution on $\R^d$ with unknown mean $\mu = \Ep_P[X]$ and unknown
covariance
$\Sigma = \Ep_P[(X-\mu)(X-\mu)^T]$, and let $X_i \simiid P$, $i \le n$.
For an estimator $\est{\mu}$, consider the
covariance-normalized error $\mbox{err}_\Sigma(\est{\mu}, \mu) \defeq
(\est{\mu} - \mu)^T \Sigma^{-1} (\est{\mu} - \mu)$. 
We give an $(\diffp, \delta)$-differentially private estimator $\msout$ of $\mu$ such that,
assuming the vectors $\Sigma^{-1/2} X_i$ are sub-Gaussian and $n = \wt{\Omega}(d/\diffp^2)$,
\begin{equation}
  \label{eqn:we-got-it}
  \mbox{err}_\Sigma(\msout, \mu)
  = (\msout - \mu)^T \Sigma^{-1} (\msout - \mu)
  \le \wt{O}(1) \left[\frac{d + \log\frac{1}{\delta}}{n}
    + \frac{d^2 \log^2 \frac{1}{\delta}}{n^2 \diffp^2}
    \right]
\end{equation}
with probability at least $1 - \delta$,
where the $\wt{O}(1)$ term hides dependence on the sub-Gaussian parameter of
$\Sigma^{-1/2} X$ and logarithmic factors in $n$. Except for a factor of
$\log\frac{1}{\delta}$ and the hidden logarithmic factors in $n$,
this is
optimal, and the method extends naturally to distributions with heavier
tails for which we can provide similar near-optimal guarantees.

By measuring error with respect to the covariance $\Sigma$ of the data
itself, we adopt the familiar efficiency goals of classical theoretical
statistics: that an estimator should be adaptive to structure in covariates
and should have (near)-optimal covariance. Mean estimation is, of course,
one of the most basic problems in statistics, and we have
known for seventy-odd years that the
sample mean $\wb{X}_n \defeq \frac{1}{n} \sum_{i = 1}^n X_i$
is efficient~\cite{Cramer46,LeCam86}, achieving the
optimal error
$\Ep[(\wb{X}_n - \mu)^T \Sigma^{-1} (\wb{X}_n - \mu)] = \frac{d}{n}$,
with high-probability guarantees under appropriate
moment assumptions~\cite{Wainwright19}. Perhaps stating the obvious,
the sample mean is adaptive to the covariance of the distribution: no matter
$\Sigma$, the sample mean is efficient.

When we require estimators to be private, however, the story is less
clear. While differential privacy~\cite{DworkMcNiSm06, DworkKeMcMiNa06} has
become the \emph{de facto} choice for protecting sensitive data in the
sixteen or so years since its release---with substantial theoretical
advances and successful applications~\cite{ErlingssonPiKo14, ApplePrivacy17,
  AbadiChGoMcMiTaZh16, DajaniLaSiKiReMaGaDaGrKaKiLeScSeViAb17,
  GarfinkelAbPo18, DworkRo14}---we know of no computationally efficient
procedures that achieve order-optimal sample complexity
with respect to the natural Mahalanobis norm $\norm{v}_\Sigma = \sqrt{v^T
  \Sigma^{-1} v}$ the population $P$ induces via its
covariance.  \citet{BrownGaSmUlZa21} highlight this, developing sample
efficient procedures that achieve small error in the Mahalanobis metric even
when $\Sigma$ is unknown. When the covariance $\Sigma$ is known, estimators
that truncate the data relative to $\Sigma$ and add Gaussian noise to such a
trimmed mean with covariance proportional to $\Sigma$ suffice to privately
estimate $\mu$ (under approximate differential privacy) with the essentially
optimal rate~\eqref{eqn:we-got-it}, so that $n = \Omega(d)$ observations
suffice to estimate $\mu$ (see, e.g.,~\cite{BiswasDoKaUl20,
  BrownGaSmUlZa21}).  But in the more realistic setting that $\Sigma$ is
unknown, to the best of our knowledge all prior work either requires a
sample of size $n =
\Omega(d^{3/2})$; is intractable, taking time exponential in $n$ or
$d$ to compute; or assumes $P$ is isotropic.  Many of these further assume
$P$ is Gaussian, a stringent assumption that never obtains in
practice.  See
Section~\ref{sec:related-work} for more discussion.

Our contribution is a polynomial-time private estimator
(Algorithm~\ref{algorithm:privmean}, $\privmean$) whose error matches the error
achievable when the covariance is known (equivalently, the data is
isotropic) to polylogarithmic factors. In essence, our estimator privatizes
a stable estimate of the empirical mean by adding Gaussian noise with
covariance proportional to a stable estimate of the empirical covariance; it
takes time $\tilde{O}(nd^2)$ to compute, has (nearly) linear sample
complexity for sub-Gaussian distributions, allows $\Sigma$ to be degenerate
or low-rank, and naturally extends beyond sub-Gaussianity.

\subsection{Related work}\label{sec:related-work}



There are many connections between differential privacy and robust
statistics~\cite{DworkLe09}, in that the major focus of robust statistics is
to develop estimators insensitive to outliers and corrupted
data~\cite{Tukey60, Huber64, HuberRo09, HampelRoRoSt86}, while differential
privacy makes the output (distributions) of estimators similar even when
individuals in the underlying data change~\cite{DworkMcNiSm06,
  DworkKeMcMiNa06, DworkLe09}.  While Tukey and Huber's initiation of robust
statistics is more than sixty years old~\cite{Tukey60, Huber64}, studying
statistical limits of estimation and inference from corrupted data,
computational tractability was elusive: only in the last decade have
researchers developed computationally efficient methods for even robustly
estimating a sample mean~\citep{DiakonikolasKa22}.
Similarly, only recently has the community elucidated
trade-offs between statistical and computational considerations in
robust estimation~\citep{DiakonikolasKa22}.


It is natural to wonder whether such trade-offs also arise with privacy.
For example, classical procedures in private query evaluation require
exponential time in natural problem parameters~\cite{HardtRo10,
  DworkRoVa10}.  Likewise, in estimation, following the ``propose, test,
release'' framework of \citet{DworkLe09}, a number of sample efficient
private estimators~\cite{LiuKoOh22, BrownGaSmUlZa21, HopkinsKaMaNa22}
require testing whether a given statistic is robust to the removal of groups
of data points, which can be computationally intractable in high-dimensions.
In a number of these settings, computationally efficient estimators
achieving comparable sample efficiency have emerged only within the last
year or
so~\citep[e.g.][]{HopkinsKaMaNa22, KothariMaVe21, AshtianiLi22, AlabiKoTaVeZh22}.
Our mean estimation setting is a striking example of a seemingly
simple problem for which no known sub-exponential time and sample efficient
algorithm exists. In particular, to the best of our knowledge, all previous
work
has either (i) exponential runtime~\citep{BrownGaSmUlZa21, LiuKoKaOh21}; (ii) is sample inefficient~\citep{KamathLiSiUl18, LiuKoKaOh21}, requiring sample
size at least $n = \Omega(d^{3/2})$;
or (iii) otherwise essentially assumes the population covariance $\Sigma$ is
isotropic~\citep{KamathLiSiUl18, BiswasDoKaUl20, HuangLiYi21,
  LiuKoKaOh21} (nominally, the paper~\cite{HuangLiYi21}
allows arbitrary covariance, but the squared error of its estimator
scales at least linearly with the condition number of the
population covariance $\Sigma$, which is effectively equivalent to
assuming isotropic covariance~\cite{BiswasDoKaUl20}).
Here we have highlighted the most relevant (recent) examples; see
in the paper~\citep{BrownGaSmUlZa21} for
coverage of earlier work.

The work most closely related to ours is that of \citet{BrownGaSmUlZa21},
who also consider covariance-adaptive
mean estimation and also achieve (nearly) linear
sample complexity. They give a roadmap to adaptive private mean estimation
that circumvents private covariance estimation, a task whose sample
complexity is necessarily $\Omega(d^{3/2})$ (see the lower bound by
\citet{DworkTaThZh14}), and are the first to achieve sample complexity
$o(d^{3/2})$, let alone linear.  However, their estimators take exponential
time to compute; moreover, while their accuracy analysis is independent of
the condition number of $\Sigma$, it assumes $\Sigma$ is full rank.
Finally, they only consider Gaussian and sub-Gaussian distributions.

In concurrent work,
\citet{HopkinsKaMaNa22} give a generic reduction from private estimation to
robust estimation and leverage
this reduction to obtain private estimators with (near) optimal sample
complexity.  While their reduction is generic, the resulting estimators are
efficient only in certain special cases, e.g., for Gaussian distributions
whose algebraic moment relationships allow efficient formulation, and their
results for mean estimation assume bounded covariance.  They extend a line
of work~\cite{KothariMaVe21,HopkinsKaMa22} on
obtaining efficient approximations of inefficient private mechanisms via
sum-of-squares (SoS) relaxations.  While technically efficient, SoS
estimators typically incur large polynomial runtime and thus scale
poorly to high-dimensional settings or large amounts of data. Unlike our
estimator, however, they are robust to corruption of a constant fraction
the data.

\subsection{Organization}

We provide a brief outline of the paper to come.  Section~\ref{sec:prelims}
introduces notation and covers the preliminary privacy definitions we
require for our development.  Our main estimator, $\privmean$, consists of
two main parts: stably estimating the covariance of the data to reasonable
accuracy and then estimating a truncated mean to which we add noise. We
present our algorithms in Section~\ref{sec:main-results}, where
Section~\ref{sec:covsafe-alg} gives the covariance estimator,
Section~\ref{sec:meansafe-alg} the mean estimator, and
Section~\ref{sec:privmean-alg} presents the full procedure; we analyze
$\privmean$'s privacy in Section~\ref{sec:main-privacy-result}, deferring
some of the requisite proofs to Sections~\ref{sec:covariance} and
\ref{sec:mean}.  We provide accuracy analysis in Section~\ref{sec:accuracy},
where we also present $\adamean$ (Algorithm~\ref{algorithm:adamean}), which
allows $\privmean$ to adapt to the scale of the observed data.

\section{Preliminary definitions, privacy properties, and mechanisms}
\label{sec:prelims}

To make our coming development smoother and easier,
here we introduce notation and recapitulate the privacy definitions we
use throughout. We also review a few standard privacy mechanisms,
providing guarantees on their behavior; for those results that are not
completely standard, we include proofs in the appendices for completeness.

\subsection{Notation}\label{sec:notation}

\paragraph{Semidefinite matrices and norms}

\newcommand{\Proj}{\Pi}  

For a positive semidefinite (PSD) matrix $\genpsd \in \R^{d\times d}$, we
let $\colspace(\genpsd)$ denote its columnspace and
$\genpsd^\dag$ its pseudoinverse,
while the square-root of the pseudoinverse is
$\genpsd^{\dag/2}$.  We let $\Proj_{A}
\defeq A^\dagger A
= A^{\dagger/2} A^{1/2} \in
\R^{d\times d}$ denote the orthogonal projector onto $\colspace(\genpsd)$.
Using the nuclear norm $\nucnorm{A} =
\sum_{i = 1}^n \sigma_i(A)$ (the sum of $A$'s singular values),
we define the distance-like quantity for PSD matrices $A, B$ as
\begin{equation*}
  \dpsd(A, B) \defeq
  \begin{cases}
    \max\brc{
      \nucnorm{A^{\dag/2} (B - A) A^{\dag/2}},
      \nucnorm{B^{\dag/2} (A - B) B^{\dag/2}}
    } & \mbox{if}~\colspace(A) = \colspace(B) \\
    \infty & \text{otherwise},
  \end{cases}
\end{equation*}
setting $\dpsd(A, B)=\infty$ if $A$ or $B$ are not PSD.
When $A$ and $B$ are invertible, $\dpsd(A, B) = \max\{\nucnorms{A^{-1/2} B
  A^{-1/2} - I}, \nucnorms{B^{-1/2} A B^{-1/2} - I}\}$, though we note in
passing that it is not a distance.  The extended-value Mahalanobis
norm $\norm{\cdot}_\genpsd$ corresponding to $\genpsd \succeq 0$ is
\begin{equation*}
  \norm{v}_\genpsd^2 \defeq
  \lim_{t \downarrow 0}
  v^T(\genpsd + t I)^{-1} v
  =
  \begin{cases}
    v^T \genpsd^\dag v & v \in \colspace(\genpsd) \\
    +\infty & \text{otherwise}.
  \end{cases}
\end{equation*}
When $\genpsd$ is non-singular, this is the standard
$\norm{v}_\genpsd = \sqrt{v^T A^{-1} v}$,
and the norm has the monotonicity property that
if $\genpsd \preceq B$, then $\norm{v}_\genpsd \geq \norm{v}_{B}$
for all $v \in \R^d$.

\paragraph{Sets and Partitions}
For sets $S,S'$, define the distance
$\dsym(S, S') \defeq \max\{\cardp{S \setminus S'}, \cardp{S' \setminus S}\}$.
Given integers $n$ and $b$, where we assume $b$ divides $n$ for simplicity,
we let $\Partitions_{n,b}$ be the set of all
partitions of $[n]$ such that each subset constituting the partition has
$b$ elements.
We represent a given partition in $\Partitions_{n,b}$ as a tuple of subsets
$\partition = (\partitionset_1, \ldots, \partitionset_{n/b})$,
where each $\partitionset_j \subset [n]$ has $b$ elements
and are pairwise disjoint.

\paragraph{Distributions}
We let $W \sim \laplace(\sigma)$ denote that $W$ has Laplace distribution
with scale $\sigma$, with density $p(w) = \frac{1}{2\sigma} \exp(-|w|
/ \sigma)$. $X \sim \normal(\mu, \Sigma)$ indicates that $X$ is normal
with mean $\mu$ and covariance $\Sigma \succeq 0$,
where if $\Sigma$ is not full rank we mean that $X$ has support
$\mu + \colspace(\Sigma)$.

\subsection{Privacy definition and basic properties}

It will be convenient for us to use closeness of distributions
in our derivations (cf.~\cite[Ch.~3.5]{DworkRo14}), so we
frame differential privacy as a type of closeness in distribution.


\begin{definition}[$(\eps,\delta)$-closeness]
  \label{def:closeness}
  Probability distributions
  $P$ and $Q$ are \emph{$(\eps, \delta)$-close in distribution},
  denoted $P \closeed Q$,
  if for all measurable sets $A  \subset \cX$,
  \begin{equation*}
    P(A) \le e^\eps Q(A) + \delta
    ~~ \mbox{and} ~~
    Q(A) \le e^\eps P(A) + \delta.
  \end{equation*}
  Similarly, random variables $X$ and $Y$ are $(\eps,\delta)$-close,
  $X \closeed Y$, if their
  induced distributions are: $\P(X \in \cdot) \closeed \P(Y \in \cdot)$
\end{definition}
\noindent
Differential privacy~\cite{DworkMcNiSm06,DworkKeMcMiNa06} is then
equivalent to this notion of closenss:
a randomized function (or mechanism) $M$ from an input space $\mc{X}^n$ to
$\mc{Y}$ is then \emph{$(\diffp, \delta)$}-differentially private if and
only if for any vectors $x, x' \in \mc{X}^n$ differing in only a single
element,
\begin{equation*}
  M(x) \closeed M(x').
\end{equation*}

The following results on closeness are standard~\cite[Ch.~3]{DworkRo14}.

\begin{lemma}[Basic composition] 
  \label{lemma:joint-closeness}
  Let $X, X', Y, Y'$ be random variables satisfying
  $X \close{\eps_X}{\delta_X} X'$, and
  $Y \close{\eps_Y}{\delta_Y} Y'$.
  Then $(X, Y) \close{\eps_X+\eps_Y}{\delta_X+\delta_Y} (X', Y')$.
\end{lemma}

\begin{lemma}[Group composition]\label{lemma:group-privacy}
  Let $X_1, \ldots, X_k$ be random variables with $X_i
  \close{\eps_i}{\delta_i} X_{i + 1}$ for each $i$.  Let $\diffp_{> i}
  \defeq \sum_{j = i + 1}^{k-1} \diffp_j$, $\diffp = \sum_{i = 1}^{k-1}
  \diffp_i$, and $\delta = \sum_{i = 1}^k e^{\diffp_{> i}} \delta_i$. Then
  $X_1 \close{\diffp}{\delta} X_k$.
\end{lemma}

\begin{lemma}[Post-Processing]\label{lemma:post}
  Let $X, Y, W$ be random variables. Then
  for any function $f$,
  if $X \closeed Y$, then $f(X, W) \closeed f(Y, W)$.
\end{lemma}

\subsection{Mechanisms}

We use several known mechanisms, and our
procedures rely on their distributional closeness properties.  The first is
the $\topkdp$ mechanism, which (approximately) returns the largest $k$
elements of a sample. In our analysis, it will be convenient to call the
procedures we develop with noise as an argument to allow easier
tracking of distributional closeness.

\begin{algorithm}
  \DontPrintSemicolon
  \SetKwInOut{Input}{Input}
  \SetKwInOut{Output}{Output}
  \SetKwInOut{Params}{Params}
  \SetKwInOut{Noise}{Noise}
  \caption{Top-$k$ DP ($\topkdp$)
    \label{algorithm:topk-dp}}
  \Input{ data $x \in \R^p$, threshold $k$ }
  \Noise{ $\xi_1, \xi_2 \in \R^{p}$ }
  \Output{ $R \subseteq [n]$ such that $\cardp{R} = k$, $\tilde{x} \in (\R \cup \{\bot\})^p$ }

  $y_1 \leftarrow x + \xi_1$\;
  $y_2 \leftarrow x + \xi_2$\;

  $R \leftarrow$ index set comprising the $k$ largest $y_{1,j}$'s\;
    \label{line:topk_R-def}

  \For{$j \in [p]$}{
    \uIf{$j \in R$}{ \label{line:topk_m-check}
      $\tilde{x}_j \leftarrow y_{2,j}$\;
    }
    \Else{
      $\tilde{x}_j \leftarrow \bot$\; \label{line:topk_bot-ass}
    }
  }

  \Return $\tilde{x}$\;
\end{algorithm}

\begin{lemma}[$\topkdp$ mechanism, \cite{QiaoSuZh21}, Theorem 2.1]
  \label{lemma:topk-closeness}
  Let $\gamma, \diffp \in \R_+$. Let $x, x' \in \R^p$ be
  such that $\linf{x - x'} \leq \gamma$.
  Then for
  $\xi_1, \xi_2 \sim \laplace(\frac{2k\gamma}{\diffp})^{p}$,
  \begin{align*}
    \topkdp(x, k; \xi_1, \xi_2) \close{\diffp}{0} \topkdp(x', k; \xi_1, \xi_2).
  \end{align*}
\end{lemma}



\newcommand{\rendiv}[2]{D_\alpha\left({#1} |\!| {#2}\right)}

\noindent
As our procedures rely on adding Gaussian noise, we require two
distributional closeness results for normal distributions.  See
Appendices~\ref{sec:proof-normal-closeness-mean}
and~\ref{sec:proof-normal-closeness-cov} for proofs, which
we include for completeness, as they are both tweaks of
existing results~\cite{DworkRo14, Mironov17}.

\begin{lemma}[Gaussians, distinct means]
  \label{lemma:normal-closeness-mean}
  Let $\mu_1, \mu_2 \in \R^d$ and let $\Sigma \in \R^{d \times d}$ be PSD.
  Suppose $\lv \mu_1 - \mu_2 \rv_{\Sigma} \leq \rho$ and
  define
  \begin{equation*}
    \tau
    = \begin{cases}
      \frac{\rho}{\diffp}
      \sqrt{2\log\frac{5}{4 \delta}}
      & \mbox{if~} 0 < \eps \le 1 \\
      \rho / \left(\sqrt{2 \log \frac{1}{\delta} + 2 \diffp} - \sqrt{2
        \log \frac{1}{\delta}}\right) & \mbox{otherwise}.
    \end{cases}
  \end{equation*}
  Then
  $\normal(\mu_1, \tau^2 \Sigma) \closeed \normal(\mu_2, \tau^2 \Sigma)$.
\end{lemma}
\noindent
\citet[Lemma 4.15]{BrownGaSmUlZa21} essentially
give the next result, but
we allow low rank covariance matrices.
\begin{lemma}[Gaussians, distinct covariances]
  \label{lemma:normal-closeness-cov}
  Let $\mu \in \R^d$ and $\Sigma_1, \Sigma_2 \in \R^{d\times d}$ be
  PSD and satisfy
  $\dpsd(\Sigma_1, \Sigma_2) \leq \gamma < \infty$.
  Then $\normal\left(\mu, \Sigma_1\right) \closeed
    \normal\left(\mu, \Sigma_2\right)$
  for $\eps \geq 6\gamma \log(2/\delta)$.
\end{lemma}

We conclude with a standard guarantee for Laplacian random
vectors~\cite[e.g.][Thm.~3.6]{DworkRo14}.
\begin{lemma}[Laplace mechanism]
  \label{prop:useful-4}
  Let $\alpha,\beta > 0$ and $\clnznormr \simiid \laplace(\beta/\alpha)$. Then for any $\set{A} \subseteq \R^m$
  and $\eta \in \R^{m}$ such that $\norm{\eta}_1 \leq \beta$,
  \begin{align*}
    \P(\clnznormr \in \set{A})
    \leq \exp(\alpha)\P(\clnznormr \in \set{A} + \eta).
  \end{align*}
\end{lemma}

\newcommand{\kmax}{k_{\max}}
\newcommand{\mmax}{m_{\max}}

\newcommand{\removedinds}{R}
\newcommand{\Transcript}{\Gamma}

\section{Algorithms}
\label{sec:main-results}



As our estimator and its full analysis are fairly involved, we provide a
broad overview of our procedures here.  We compute the estimator, whose full
treatment we give in Algorithm~\ref{algorithm:privmean} ($\privmean$) in
section~\ref{sec:privmean-alg}, in two phases, consisting of a stable
covariance estimate and a stable mean estimate. Each carefully prunes
outliers from the data, using plug-in quantities from the remaining
observations as substitutes for the usual plug-in mean and covariance.

In the first phase (Algorithm~\ref{algorithm:clip-and-repeat},
$\clipr$), we obtain a robust but non-private estimate $\est{\Sigma}$
of the covariance.  Assuming for convenience $n$ is even, we pair
observations and initially let
\begin{align*}
  \est{\Sigma} \defeq
  \frac{1}{n} \sum_{i=1}^{n/2} (x_i - x_{n/2+i})(x_i - x_{n/2+i})^T.
\end{align*}
As $x_i - x_{n/2+i}$ is symmetric, we can prune pairs of observations for
which $\norms{x_i - x_{n/2+i}}_{\est{\Sigma}}$ is large (regardless of the
population mean $\mu$), recompute $\est{\Sigma}$ on the remaining
observations, and repeat until convergence. The key is that this pruning,
while it provides no formal robustness guarantees, is stable to changes of a
single example $x_i$, ensuring $\est{\Sigma}$ itself is stable.

In the second phase (Algorithm~\ref{algorithm:meansafe}, $\meansafe$), we
first obtain a robust estimate $\est{\mu}$ of the empirical mean by trimming
outliers with respect to $\norm{\cdot}_{\est{\Sigma}}$.  Using
$\norm{x_i}_{\est{\Sigma}}$ to determine whether $x_i$ is influential for a
mean estimate is unreliable, as the quantity may be arbitrarily large even
for non-outliers if $\norm{\mu}_{\est{\Sigma}}$ itself is large;
unfortunately, paired observations (as in the stable covariance estimation
phase) are similarly unhelpful, as $\norm{x_i - x_j}_{\est{\Sigma}}$ could
be small if both $x_i, x_j$ are ``outlying'' in the same way.  Instead, we
randomly partition the $n$ observations into groups $S$ of size
$O(\log\frac{n}{\delta})$ and prune \textit{all} observations in a
group $S$ if \textit{any} two observations in $S$ are far
with respect to $\norm{\cdot}_{\est{\Sigma}}$, so that there is at least a
pair of outlying observations in the group. Assuming the total number of
pruned observations across both phases is not too large---and much of our
analysis shows how to make the pruned observations stable across different
samples $x, x'$---we let $\est{\mu}$ be the empirical mean of the un-pruned
observations, then release $\pmout \sim \normal (\est{\mu},
\sigma^2(\eps,\delta) \est{\Sigma})$, where the privacy budget determines
$\sigma^2(\eps,\delta)$.




\subsection{Stable covariance estimation}\label{sec:covsafe-alg}

The first component of the private mean estimation algorithm is the covariance
estimation procedure $\clipr$ in Alg.~\ref{algorithm:clip-and-repeat}, which
removes suitably unusual pairs of data points from the sample $x \in
(\R^d)^n$, then uses the remaining pairs to actually construct the
covariance. The procedure
maintains an empirical covariance $\Sigma_t$ of the remaining data at each
iteration $t$, so that $\{\Sigma_t\}$ is a non-increasing (in the
semidefinite order) sequence of matrices, and stores removed indices in an
iteratively growing collection $\removedinds_t$ for $t = 1, 2, \ldots$; the
procedure thus necessarily terminates after at most $n/2$ rounds of index
removal. For convenience of our analysis, $\clipr$ also returns a transcript
$\Transcript$ of the removed indices and iteratively constructed
covariances, returning $\perp$ if the data is so unstable that it removes
too many indices.


\begin{algorithm}[h]
  \DontPrintSemicolon
  \SetKwInOut{Input}{Input}
  \SetKwInOut{Output}{Output}
  \SetKwInOut{Params}{Params}
  \SetKwInOut{Noise}{Noise}
  \caption{\label{algorithm:clip-and-repeat}
    Robust Covariance Estimation ($\clipr$)}
  \Input{ data $x_{1:n}$ }
  \Params{ threshold $B$, threshold $m$ }
  \Noise{ $\clnznormf \in \R^{n/2+1}$, $\clnzthrf \in \R$ }
  $\clx \leftarrow x_{1:n/2} - x_{n/2+1:n}$\; \label{line:x-transform}

  $\removedinds_0 \leftarrow \emptyset$, $\Sigma_0 \leftarrow \frac{1}{n} \sum_{i = 1}^{n/2} \clx_{i} \clx_{i}^T$\;
  $\mathtt{converged} \leftarrow \textbf{false}$, $t \leftarrow 0$\;
  \While{\textbf{\txt{not}} $\mathtt{converged}$}{ \label{line:clipr-while}
    $t \leftarrow t + 1$, $\removedinds_t \leftarrow \removedinds_{t-1}$,
    $\mathtt{truncated} \leftarrow 0$\;
    \For{$i \in [n/2] \setminus \removedinds_{t-1}$}{
      \If{$2 \log \norm{\clx_i}_{\Sigma_{t-1}} + \clnznormf_{i} + \clnznormf_{n/2+1} > \log\(B\)$}{\label{line:clipr-truncate}
        $\removedinds_t \leftarrow \removedinds_t \cup \{i\}$\;
        $\mathtt{truncated} \leftarrow \mathtt{truncated} + 1$
      }
    }
    \If{$\mathtt{truncated} = 0$}{\label{line:clipr-convergence}
      $\mathtt{converged} \leftarrow \textbf{true}$, $T \leftarrow t$\;
    }
    $\Sigma_t \leftarrow \frac{1}{n} \sum_{i \in [n/2] \setminus \removedinds_t} \clx_i \clx_i^T$ \label{line:clipr-set-sigma}
  }
  $\Transcript \leftarrow (\[\removedinds_i\]_{t=0}^T,\[\Sigma_t\]_{t=0}^T,T )$\;
  \If{$\cardp{\removedinds_T} > m + \clnzthrf$}{\label{line:clipr-check}
    \Return $\perp,\Transcript$
  }
  \Return $\Sigma_T,\Transcript$ \label{line:clipr-return}\;
\end{algorithm}

The key is that the covariance estimates are appropriately stable
(see Conditions~\ref{item:internal-stability} and~\ref{item:external-stability}
to come in Section~\ref{sec:main-privacy-result}),
and with high probability on any given input $x$, the algorithm guarantees
that its output changes little when we remove index $i$ or, if the data has
too much variance relative to itself, that the procedure simply returns
$\est{\Sigma} = \perp$. To allow cleaner description of the precise results we
require in our main privacy result in Section~\ref{sec:main-privacy-result},
for a putative bound $B$ on $\norm{x_i - x_j}_\Sigma^2$,
acceptable number of outliers $m$,
and privacy random variables $\clnznormr$ and
$\clnzthrr$ to be specified, let
\begin{subequations}
  \label{eqn:clipr-output}
  \begin{equation}
    \label{eqn:clipr-main-output}
    (\est{\Sigma}, \Transcript) \defeq \clipr_{\clparams}(x; \clnznormr,
    \clnzthrr),
  \end{equation}
  where
  $\Transcript = ([\Sigma_t]_{t \le T}, [\removedinds_t]_{t \le T},T)$
  is the transcript of intermediate covariances and removed indices, and
  for $\clx = x_{1:n/2} - x_{n/2+1:n}$ (as in Line~\ref{line:x-transform})
  define the leave-one-out covariance
  \begin{equation}
    \label{eqn:clipr-loo-output}
    \est{\Sigma}_{-i} \defeq
    \begin{cases}
      \est{\Sigma} - \frac{1}{n}
      \indic{i \in \removedinds_T} \clx_i \clx_i^T
      & \mbox{if~} \est{\Sigma}\neq \perp \\
      \perp & \mbox{otherwise},
    \end{cases}
  \end{equation}
\end{subequations}
which is $\est{\Sigma}$ whenever $\clipr$ does not remove index
pair $(i, n/2 + i) \in [n]^2$.


\subsection{Stable mean estimation}\label{sec:meansafe-alg}

The second component of the private mean estimation algorithm is a
sample mean estimator, adding noise commensurate with an estimated
(positive semidefinite) noise covariance that we abstractly call $A \in
\R^{d \times d}$. The procedure $\meansafe$ removes elements $x_i$ of the
data $x$ that are ``too far'' from the bulk of the data, measured by
$\norm{x_i - x_{i'}}_A$, using randomization to be sure that the removed
indices are appropriately private. The algorithm uses
$\topkdp$ to select groups of indices
that contain too many outlying datapoints, then removes all data associated
with these groups.  By evaluating (random) groups of data, the
procedure enforces privacy in that if
the majority of the data are appropriately close to a center point as
measured by covariance, then few groups have large diameter, and adding
or removing a single datapoint $x_i$ can only effect the removal of one
group and the method may privately
return a noisy empirical mean. When many
datapoints are outliers, the method is likely to return $\perp$ regardless
of the behavior of any individual datapoint.

\begin{algorithm}[h]
  \DontPrintSemicolon
  \SetKwInOut{Input}{Input}
  \SetKwInOut{Output}{Output}
  \SetKwInOut{Params}{Params}
  \SetKwInOut{Noise}{Noise}
  \Input{ data $x_{1:n}$,
    PSD matrix $\genpsd \in \R^{d \times d}$ }
  \Params{ threshold $B$, batchsize $b$, threshold $k$
    }
  \Noise{ $\msnzgf = (\partitionset_1, \ldots, \partitionset_{n/b})
    \in \Partitions_{n,b}$,
    $\msnztopif, \msnztopiif \in \R^{n/b}$,
    $\msnzthrf \in \R$,
    $\msnznf \in \R^d$ }
  \Output{ mean estimate $\msout$
  }
  \caption{\label{algorithm:meansafe} Robust Mean Estimation ($\meansafe$)}

  \For{$j \in [n/b]$}{
    $D_j \leftarrow \log (\diam_\genpsd(x_{\partitionset_j}))$\;
      \label{line:ms_m-ass}
  }

  $\wt{D} \leftarrow \topkdp(D, k; \msnztopif, \msnztopiif)$\;
  \label{line:ms_topk-call}

  $\removedinds \leftarrow \emptyset$,
  $t \leftarrow 0$ \tcc*[f]{initialize removed indices to empty}\;

  \For{$j \in [n/b]$}{ \label{line:ms_rejection-for}
    \If{$\wt{D}_j \neq \bot$ \textbf{and}
      $\wt{D}_j > \log(\sqrt{B}/4)$}{ \label{line:ms_m-check}
      $\removedinds \leftarrow \removedinds \cup \partitionset_j$\;
        \label{line:ms_s-union}
      $t \leftarrow t+1$\;
    }
  } \label{line:ms_end-rejection-for}

  $\msmean \leftarrow \frac{1}{n-\cardp{\removedinds}}
  \sum_{i \not\in \removedinds} x_i$\; \label{line:ms-mean-def}

  \uIf{$t > \frac{2k}{3} + \msnzthrf$}{ \label{line:ms_thres-comp}
    $\msout \leftarrow \bot$\; \label{line:ms_bot-output}
  }
  \Else{
    $\msout \leftarrow \msmean + A^{1/2} \msnznf$\; \label{line:ms_vec-output}
  }
  $\Transcript \leftarrow \big(D, \wt{D}, \removedinds, t, \msmean\big)$\;
    \label{line:ms-transcript}
  \Return $\msout, \Transcript$\;

\end{algorithm}

For use in Section~\ref{sec:main-privacy-result}, as with $\clipr$,
we assign notation to the outputs of $\meansafe$.  Let
$x \in \R^{n \times d}$ be an arbitrary sample and $\genpsd$ an arbitrary
positive semidefinite matrix. For parameters
defining the supposed bound $B$ on $\norm{x_i - x_j}_\Sigma^2$,
group size $b$, acceptable outlier count $k$,
and privacy random
variables $(\msnzgr, \msnztopir, \msnztopiir, \msnzthrr, \msnznr)$, all to be
specified later, define
\begin{equation}
  \label{eqn:meansafe-output}
  (\wt{\mu}(x, \genpsd), \Transcript(x, \genpsd))
  \defeq \meansafe_{\msparams}(x, \genpsd;
  \msnzgr, \msnztopir, \msnztopiir, \msnzthrr, \msnznr).
\end{equation}

\subsection{The private mean estimation algorithm}\label{sec:privmean-alg}

Given $\clipr$ and $\meansafe$, Algorithm~\ref{algorithm:privmean}
($\privmean$) combines the two (with appropriate parameter settings) to
perform private mean estimation.  First, $\privmean$ computes a stable
covariance estimate via $\clipr$, and assuming the returned covariance
estimate $\est{\Sigma} \neq \perp$, then computes a trimmed mean to which it
adds Gaussian noise with covariance proportional to $\est{\Sigma}$ using
$\meansafe$. Theorem~\ref{thm:pm-privacy} in
Section~\ref{sec:main-privacy-result} shows that the parameter choices
guarantee privacy.

\begin{algorithm}
  \DontPrintSemicolon
  \SetKwInOut{Input}{Input}
  \SetKwInOut{Output}{Output}
  \SetKwInOut{Params}{Params}
  \caption{\label{algorithm:privmean}
    Covariance Adaptive Private Mean Estimation ($\privmean$)}
  \Input{ data $x_{1:n}$ }
  \Params{ threshold $B$, privacy budget $(\eps,\delta)$ }
  \Output{ mean estimate $\msout$ }

  \phase{Robust Covariance Estimation}
  $m \leftarrow \frac{16}{\diffp}\log \frac{1}{\delta}$,
  $\mmax \leftarrow m + \frac{16}{\diffp} \log \frac{1 + e^{\diffp/4}}{\delta}$
  \label{line:pm-m-setting} \;
  $\sigma_Z \leftarrow
  \frac{32 \sqrt{e} B (\mmax+1)}{n \diffp}$,
  $\Wcovscale \leftarrow \frac{16}{\diffp}$
  \label{line:cl-noise-scale-setting} \;
  $\pmclnznormr \sim \laplace(\sigma_Z)^{n/2+1}$,
  $\pmclnzthrr \sim \laplace(\Wcovscale)$
  \label{line:pm-cl-noise-setting} \;

  $\clout, \Transcript_{\clcode} \leftarrow \clipr_{\clparams}(x; \pmclnznormr, \pmclnzthrr)$\; \label{line:pm_y-ass}
  
  \If{$\clout = \bot$}{
    \Return $\bot$\;
  }

  \phase{Private Mean Estimation}
  $b \leftarrow 1 + \log_2\frac{6 n^2}{\delta}$,
  $k \leftarrow \frac{24}{\diffp}\log\frac{3}{\delta} - 3$\;
  $\mstopkscale \leftarrow \frac{8 k}{n \diffp}
  \frac{B \sqrt{e}}{1 - B \sqrt{e} / n}$,
  $\msnscale \leftarrow
  \frac{20b\sqrt{B}}{n \diffp}
  \exp(3 \mstopkscale \log \frac{12n}{b \delta})$,
  $\Wmeanscale \leftarrow \frac{8}{\diffp}$
  \;
  \label{line:pm_ms-scale-def}
  $\msnzgr \sim \uniform(\Partitions_{n, b})$,
  $Z\subtop, Z\subtop' \simiid \laplace(\mstopkscale)^{n/b}$,
  $W \sim \laplace(\Wmeanscale)$,
  $\msnznr \sim \normal(0, \msnscale^2 I_{d\times d})$\;
    \label{line:pm_msnz}
  $\msout, \Transcript_{\mscode} \leftarrow
    \meansafe_{\msparams}(x, \clout;
      \msnzgr, Z\subtop, Z\subtop', W, \msnznr)$\;
      \label{line:pm_ms-call}
  \Return $\msout$
\end{algorithm}

We remark briefly on the runtime of $\privmean$.  Each iteration of the
\textbf{while} loop (beginning in Line~\ref{line:clipr-while}) of $\clipr$
involves a $d \times d$ matrix inversion followed by taking (at most) $n \ge
d$ matrix-vector products, requiring $O(n d^2)$.  We may modify $\clipr$
without changing its behavior to terminate after $m + \pmclnzthrr$
iterations, as rejecting more than $m + \pmclnzthrr$ indices guarantees that
$\clipr$ (and hence $\privmean$) returns $\bot$ (see
Line~\ref{line:clipr-check}). With high probability, we have $m +
\pmclnzthrr = O(\frac{1}{\diffp} \log \frac{1}{\delta})$, and giving runtime
$O(n d^2 \min\{n, \frac{1}{\diffp} \log\frac{1}{\delta}\})$.
$\clipr$'s runtime dominates $\meansafe$'s,
giving total (high probability) runtime
$O(n d^2 \min\{n, \frac{1}{\diffp} \log \frac{1}{\delta}\})$.
As an aside, we may convert
this expected runtime into a worst-case runtime of the same order
without effecting the privacy of $\privmean$ by truncating
$\pmclnzthrr$ to scale $\frac{1}{\diffp} \log \frac{1}{\delta}$.


\section{Main privacy result}
\label{sec:main-privacy-result}
The analysis of $\privmean$ is fairly involved, though there are
four key building blocks.
The first two conditions involve what we term \emph{internal} and
\emph{external} leave-one-out stability of the covariance
estimates~\eqref{eqn:clipr-main-output} and~\eqref{eqn:clipr-loo-output}
$\clipr$ returns. These conditions require that the covariance
estimates~\eqref{eqn:clipr-output} are appropriately stable,
both in terms of removing a single element contributing to the covariance
estimate $\est{\Sigma}$ on input $x$ and in terms of stability
across two inputs $x, x'$ whose transformations in Line~\ref{line:x-transform}
of $\clipr$, i.e.,
$\clx = x_{1:n/2} - x_{n/2+1:n}$ and $\clx' = x_{1:n/2}' - x_{n/2+1:n}'$,
differ only in a single element.
Letting $0 \le a < \infty$ be a constant to be determined
later and $\gamma \in (0, 1)$ be a probability, consider the conditions
\begin{enumerate}[label=(C.\roman*),labelindent=0pt]
\item \label{item:internal-stability} \emph{Internal leave-one-out
  stability.}  Let $\est{\Sigma}$ and $\est{\Sigma}_{-i}$ be the
  outputs~\eqref{eqn:clipr-output} of $\clipr$ on an arbitrary input $x$ of
  size $n$.  Then for any index $i \in [n/2]$, with probability at least $1 -
  \gamma$,
  \begin{equation*}
    \dpsd(\est{\Sigma}, \est{\Sigma}_{-i})
    \le \frac{a}{n}
    ~~ \mbox{or} ~~
    \est{\Sigma} = \perp.
  \end{equation*}
\item \label{item:external-stability}
  \emph{External leave-one-out stability.}
  Let $\est{\Sigma}$ and $\est{\Sigma}'$ be the outputs of
  $\clipr$ on inputs $x, x'$ of size $n$ such that $\clx$ and $\clx'$ differ 
  only in index $i \in [n/2]$, where
  $\est{\Sigma}_{-i}$ and $\est{\Sigma}'_{-i}$ are defined
  as in~\eqref{eqn:clipr-loo-output}. Then
  \begin{equation*}
    \est{\Sigma}_{-i} \closeed \est{\Sigma}'_{-i}.
  \end{equation*}
\end{enumerate}


The second two conditions involve the noisy truncated mean
estimate~\eqref{eqn:meansafe-output} $\meansafe$ outputs. The first of these
conditions \ref{item:mean-sample} essentially states $\meansafe$ is stable
over inputs $x$ and $x'$ differing in a single element, while the second
states that $\meansafe$ applied with identical input samples $x, x'$ but
different covariance estimates $A, A'$ is stable so long as $A,A'$ are
close in the same sense as in Condition~\ref{item:internal-stability}.

\begin{enumerate}[label=(C.\roman*),labelindent=0pt]
  \setcounter{enumi}{2}
\item \label{item:mean-sample} \emph{Mean sample stability}. Let
  $\wt{\mu}(A, x)$ be the mean $\meansafe$
  outputs~\eqref{eqn:meansafe-output} on input covariance $A$ and data $x$,
  and let $x, x'$ differ only in one element.  Then
  \begin{equation*}
    \wt{\mu}(x, A) \closeed \wt{\mu}(x', A).
  \end{equation*}
\item \label{item:mean-covariance}
  \emph{Mean covariance stability}.
  If $\dpsd(A, A') \le \frac{a}{n}$, then
  $\wt{\mu}(x, A) \closeed \wt{\mu}(x, A')$.
\end{enumerate}

Conditions~\ref{item:internal-stability}--\ref{item:mean-covariance} form
the basic privacy building blocks to show that
$\privmean$ is differentially private, and the following
proposition---a warm-up for the full Theorem~\ref{thm:pm-privacy}
to come---shows how we may relatively easily
synthesize the conditions to achieve privacy.
\begin{proposition}
  \label{proposition:condition-composition}
  Let samples $x, x'$ differ in a single element, and let $\est{\Sigma}$ and
  $\wt{\mu}(x, \est{\Sigma})$ and $\est{\Sigma}'$ and
  $\wt{\mu}(x', \est{\Sigma}')$ be the covariance and mean
  estimates~\eqref{eqn:clipr-output} and~\eqref{eqn:meansafe-output}
  for inputs $x$ and $x'$, respectively.
  Let Conditions~\ref{item:internal-stability}--\ref{item:mean-covariance}
  hold. Then
  \begin{equation*}
    \wt{\mu}(x, \est{\Sigma})
    \close{4\diffp}{(e^{3 \diffp} + e^\diffp) \delta
      + (e^{2 \diffp} + 1) (\delta + \gamma)}
    \wt{\mu}(x', \est{\Sigma}').
  \end{equation*}
\end{proposition}
\begin{proof}
  As $x$ and $x'$ are adjacent, there exists $i \in [n/2]$
  such that $\clx_{-i} = \clx_{-i}'$.
  We have a string of approximate distributional
  equalities that, together with the transitivity
  of distributional closeness implied by group privacy 
  (Lemma~\ref{lemma:group-privacy}), make the proposition immediate.
  First, we show that conditions~\ref{item:internal-stability}
  and~\ref{item:mean-covariance} imply
  \begin{equation*}
    \wt{\mu}(x, \est{\Sigma})
    \close{\diffp}{\delta + \gamma} \wt{\mu}(x, \est{\Sigma}_{-i})
    ~~~ \mbox{and} ~~~
    \wt{\mu}(x', \est{\Sigma}')
    \close{\diffp}{\delta + \gamma} \wt{\mu}(x', \est{\Sigma}_{-i}').
  \end{equation*}
  We prove the first equality as the argument for the second is identical.
  Treating $x$ as fixed, let
  $\mc{E}$ be the event that
  $\dpsd(\est{\Sigma}, \est{\Sigma}_{-i}) \le \frac{a}{n}$ or
  $\est{\Sigma} = \perp$. Then  
  for any measurable set $O$ we have
  \begin{align*}
    \P(\wt{\mu}(x, \est{\Sigma}) \in O)
    & = \Ep\left[\P(\wt{\mu}(x, \est{\Sigma}) \in O \mid \est{\Sigma})
      \indic{\mc{E}}
      \right]
    + \Ep\left[\P(\wt{\mu}(x, \est{\Sigma}) \in O \mid \est{\Sigma})
      \indic{\mc{E}^c}\right] \\
    & \stackrel{(i)}{\le}
    \Ep\left[\left(e^{\diffp} \P(\wt{\mu}(x, \est{\Sigma}_{-i}) \in O
      \mid \est{\Sigma}_{-i}) + \delta\right)
      \indic{\mc{E}}\right]
    + \P(\mc{E}^c) \\
    & \le e^\diffp \P(\wt{\mu}(x, \est{\Sigma}_{-i}) \in O) + \delta
    + \gamma,
  \end{align*}
  where inequality~$(i)$ is Condition~\ref{item:mean-covariance}
  and the final inequality follows from the $\gamma$ probability
  bound in Condition~\ref{item:internal-stability}.
  Second, we have the distributional approximations
  \begin{equation*}
    \wt{\mu}(x, \est{\Sigma}_{-i})
    \close{\diffp}{\delta}
    \wt{\mu}(x, \est{\Sigma}_{-i}')
  \end{equation*}
  by Condition~\ref{item:external-stability}, because
  post-processing preserves distributional closeness (Lemma~\ref{lemma:post}).
  Finally, we observe from the mean sample stability
  condition~\ref{item:mean-sample}
  that
  \begin{equation*}
    \wt{\mu}(x, \est{\Sigma}')
    \close{\diffp}{\delta} \wt{\mu}(x', \est{\Sigma}').
  \end{equation*}
  Combining each distributional equality, we have
  \begin{align*}
    \wt{\mu}(x, \est{\Sigma})
    \close{\diffp}{\delta + \gamma}
    \wt{\mu}(x, \est{\Sigma}_{-i})
    & \close{\diffp}{\delta} \wt{\mu}(x, \est{\Sigma}'_{-i})
    \close{\diffp}{\delta + \gamma} \wt{\mu}(x, \est{\Sigma}')
    \close{\diffp}{\delta} \wt{\mu}(x', \est{\Sigma}').
  \end{align*}
  Apply Lemma~\ref{lemma:group-privacy}.
\end{proof}

Finally, then, we come to our main privacy theorem, which
verifies that the procedures making up $\privmean$ indeed satisfy
Conditions~\ref{item:internal-stability}--\ref{item:mean-covariance} with
appropriate constants. We state the theorem here,
giving a proof that consists of lemmas making precise the constants
that appear in the conditions and whose proofs we defer.
\begin{theorem}\label{thm:pm-privacy}
  Let $B < \infty$, $\delta \in (0, 1)$,
  and let $x, x' \in (\R^d)^n$ be adjacent samples, and
  let $\diffp \le 8$.
  Define
  $\delta' = (e^{3 \diffp / 4} + e^{\diffp/4}) \delta
  + 2 (e^{\diffp/2} + 1) \delta$ and let
  $m \in \N$ be as in line~\ref{line:pm-m-setting} of
  $\privmean$.
  Assume that $\delta \le \frac{1}{n}$ and $n$ is large enough that
  \begin{equation*}
    n \ge
    \frac{128 \sqrt{e} B \log \frac{n(1 + e^{\diffp/4})}{\delta}}{\diffp}
    \left(m + 1 + \frac{16}{\diffp} \log \frac{1 + e^{\diffp/4}}{\delta}\right)
    = O(1) \frac{B \log^2 \frac{1}{\delta}}{\diffp^2}.
  \end{equation*}
  Then
  $\privmean_{B, (\diffp, \delta)}(x)$ is $(\diffp, \delta')$-differentially
  private.
\end{theorem}
\noindent
As a brief remark, the condition $\diffp \le 8$ is only for convenience;
a minor modification of the proof
allows arbitrary $\diffp$ at the expense of
a more convoluted theorem statement but in which $n$ remains of the same order.
\begin{proof}
  By Proposition~\ref{proposition:condition-composition}, it suffices to
  verify
  Conditions~\ref{item:internal-stability}--\ref{item:mean-covariance},
  where we demonstrate each holding with parameters $(\diffp/4, \delta)$.
  Throughout the proof, the value $m \in \N$ (line~\ref{line:pm-m-setting}
  in $\privmean$) and parameter $B < \infty$ remain tacit, as the
  privacy guarantee holds regardless.

  First, we consider Conditions~\ref{item:internal-stability} and
  \ref{item:external-stability} on the covariance estimates. We prove the
  coming two lemmas in Section~\ref{sec:covariance}, which begins with
  preliminaries that we require for their proofs before giving the proofs
  proper.  Our first lemma provides sufficient conditions to verify
  Condition~\ref{item:internal-stability}, internal stability.  Let $z \in
  \R^{n/2 + 1}$ and $w \in \R$ be variables---these will be random to allow
  privacy presently, but we use them for the definition---and let
  \begin{subequations}
    \label{eqn:all-instantiated-sigma}
    \begin{equation}
      \label{eqn:instantiated-sigma}
      \est{\Sigma}(x, z, w) ,\([\removedinds_t]_{t=0}^T,[\Sigma_t]_{t=0}^T,T\)
      := \clipr_{\clparams}(x; z, w),
    \end{equation}
    where we leave the dependence of the transcript
    $([\removedinds_t], [\Sigma_t], T)$ on $(x, z, w)$ implicit,
    and redefine $\est{\Sigma}_{-i}$ as in the
    definitions~\eqref{eqn:clipr-output}:
    \begin{equation}
      \label{eqn:instantiated-sigma-minus-i}
      \est{\Sigma}_{-i}(x, z, w) \defeq \est{\Sigma}(x, z, w)
      - \frac{1}{n} \indic{i \in \removedinds_T} \clx_i \clx_i^T
    \end{equation}
  \end{subequations}
  whenever $\est{\Sigma}(x,z,w) \neq \bot$, and $\bot$ otherwise.
  Then we have
  \begin{lemma}[Internal stability]
    \label{lemma:clipr-internal-stability}
    Let $Z_j \simiid \laplace(\sigma)$ and $i \in [n/2]$.
    Then with probability
    at least $1 - \exp(-\frac{1}{4\sigma})$, either
    $\est{\Sigma}(x, Z, w) = \bot$ or
    \begin{align*}
      \dpsd(\est{\Sigma}(x, Z, w), \est{\Sigma}_{-i}(x, Z, w))
      \le \frac{1}{1 - B \sqrt{e}/n}
      \frac{B \sqrt{e}}{n}.
    \end{align*}
  \end{lemma}
  \noindent
  See Section~\ref{sec:proof-clipr-internal-stability} for a proof of
  Lemma~\ref{lemma:clipr-internal-stability}.  Turning to the
  condition~\ref{item:external-stability} on external stability of $\clipr$,
  we compare the leave-one-out covariances $\est{\Sigma}_{-i}(x, \clnznormf,
  \clnzthrf)$ and $\est{\Sigma}_{-i}(x', \clnznormf, \clnzthrf)$ with input
  samples $x$ and $x'$, respectively, with identical (randomization)
  parameters $\clnznormf, \clnzthrf$. 
  Recalling $\clx = x_{1:n/2} - x_{n/2+1:n}$ and $\clx' = x_{1:n/2}' - x_{n/2+1:n}'$,
  we have the following guarantee:
  \begin{lemma}[External stability]
    \label{lemma:clipr-loo-stability}
    Let $\gamma \in (0, 1)$,
    $Z_j \simiid \laplace(\sigma_Z)$,
    $W \sim \laplace(\Wcovscale)$ and $k \in \N$. Define
    $\mmax = m + \Wcovscale \log\frac{1}{\gamma}$ and
    $\alpha = \frac{1}{\Wcovscale} + \frac{2 \sqrt{e} B (\mmax+1)}{n
      \sigma_Z}$ and $\beta = \frac{\gamma}{2}
    + \frac{n}{2} \exp(-\frac{1}{4 \sigma_Z})$.
    For all $i \in [n/2]$, if $\clx_{-i} = \clx_{-i}'$ then
    \begin{equation*}
      \est{\Sigma}_{-i}(x, Z, W)
      \close{2\alpha}{(1 + e^\alpha)\beta} \est{\Sigma}_{-i}(x', Z, W).
    \end{equation*}
  \end{lemma}
  \noindent
  The lemma effectively shows that the sets of removed indices
  $\removedinds$ and $\removedinds'$ are stable, and as they determine
  $\est{\Sigma}_{-i}$ and $\est{\Sigma}_{-i}'$, this yields their closeness. See
  Section~\ref{sec:proof-clipr-loo-stability} for a proof of
  Lemma~\ref{lemma:clipr-loo-stability}.
  
  We turn to the guarantees of $\meansafe$, realizing
  Conditions~\ref{item:mean-sample} and~\ref{item:mean-covariance}.  
  Recall the definition \eqref{eqn:meansafe-output}
  of $\msout(x, A)$ as the output of
  $\meansafe$ on input
  $x \in (\R^d)^n$ with positive
  semidefinite $A \in \R^{d\times d}$, with parameters
  bound $B$, batchsize $b$, and threshold $k \in \N$, and
  $\msnzgr$, $\msnztopir$, $\msnztopiir$, and $\msnzthrr$ as noise.
  We take $\msnztopir, \msnztopiir \in \R^{n/b}, \msnzthrr \in \R$ to be 
  Laplacian random variables, $\msnznr \in \R^d$ to be Gaussian,
  and $\msnzgr$ to be a uniformly random partition of $[n]$ into blocks of
  size $n/b$; we track their scales in giving our
  distributional approximation guarantees.

  To more cleanly state a general sample stability guarantee,
  which we may use to verify Condition~\ref{item:mean-sample},
  we define a number of additional parameters whose
  values we can determine. Let the batchsize
  $b \in \N$ and threshold $k > 0$ satisfy
  $b \ge 4$ and $2b(k+1) \leq n$. Let
  $\beta_1, \gamma \in (0, 1)$, let $\alpha \ge 0$,
  and let $\mstopkscale > 0$ and $\Wmeanscale > 0$.
  Define the constants
  \begin{equation*}
    \Delta \defeq \frac{5 b \sqrt{B}}{2n} \exp\left(3 \mstopkscale
    \log \frac{2n}{b \gamma}\right),\ 
    \beta_2 \defeq \frac{1}{2}e^{-(k/3-1)/\Wmeanscale} + \gamma + n^2 2^{1-b}
  \end{equation*}
  and
  \begin{equation*}
    \msnscale = \begin{cases}
      \frac{\Delta}{\alpha} \sqrt{2 \log \frac{5}{4 \beta_1}}
      & \mbox{if~} 0 \le \alpha \le 1 \\
      \frac{\Delta}{\sqrt{2 \log \frac{1}{\beta_1}
          + 2 \alpha} - \sqrt{2 \log \frac{1}{\beta_1}}}
      & \mbox{otherwise.}
    \end{cases}
  \end{equation*}
  With these, we have a mean-sample
  stability result from which 
  Condition~\ref{item:mean-sample} develops:
  \begin{lemma}
    \label{lemma:mean-sample}
    Let the conditions above hold and let
    $\msnztopir_j, \msnztopiir_j \simiid \laplace(\mstopkscale),
    \msnzthrr \sim \laplace(\Wmeanscale), 
    \msnznr_j \simiid \normal(0, \msnscale^2)$
    in~\eqref{eqn:meansafe-output}.
    If $x$ and $x'$ are adjacent, then
    \begin{equation*}
      \msout(x, \genpsd)
      \close{\alpha + 1/\Wmeanscale}{\beta_1 + \beta_2} \msout(x', \genpsd).
    \end{equation*}
  \end{lemma}
  \noindent
  See Section~\ref{sec:mean-sample-pf} for a proof.

  The last building block in the argument is to demonstrate
  Condition~\ref{item:mean-covariance}, that the
  estimates $\msout(x, A)$ and $\msout(x, A')$ are close when
  $A, A'$ are close. For this, we
  give the following lemma with general noise parameters.
  \begin{lemma}
    \label{lemma:mean-covariance}
    Let $b,k \in \N$, $\beta \in (0, 1)$, and
    $a, \msnscale, \alpha_2 > 0$. Define
    $\alpha_1 = \frac{6a}{n} \log \frac{2}{\beta}$,
    and define the noise scale $\mstopkscale = \frac{k a}{n \alpha_2}$.
    Then for
    $\msnztopir_j, \msnztopiir_j \simiid \Lap(\mstopkscale),
    \msnznr_j \simiid \normal(0, \msnscale^2)$,
    if
    \begin{equation*}
      \dpsd(A, A') \le \frac{a}{n}
      ~~ \mbox{then} ~~
      \msout(x, A) \close{\alpha_1 + \alpha_2}{\beta} \msout(x, A').
    \end{equation*}
  \end{lemma}
  \noindent
  See Section~\ref{sec:proof-mean-covariance} for a proof.

  For the final step, we put all the pieces together to prove the theorem.
  We give each of the lemmas so the associated condition
  (of~\ref{item:internal-stability}--\ref{item:mean-covariance}) holds with
  parameters $(\diffp/4, \delta)$, after which we can then apply
  Proposition~\ref{proposition:condition-composition} directly.  We do this
  in a somewhat odd order because of the dependence on the noise scale
  between the different lemmas, beginning with

  \paragraph{Condition~\ref{item:external-stability}.}
  For $Z_j \simiid \laplace(\sigma_Z)$ and $W \sim \laplace(\Wcovscale)$, we
  use Lemma~\ref{lemma:clipr-loo-stability} to guarantee
  Condition~\ref{item:external-stability} that $\est{\Sigma}_{-i}
  \close{\diffp/4}{\delta} \est{\Sigma}'_{-i}$.  From the lemma
  statement, we have $\est{\Sigma}_{-i} \close{2\alpha}{(1 + e^\alpha) \beta}
  \est{\Sigma}_{-i}'$, where $\alpha = \frac{1}{\Wcovscale} + \frac{2 \sqrt{e}
    B (\mmax+1)}{n \sigma_Z}$, $\beta = \frac{\gamma}{2} + \frac{n}{2}
  \exp(-\frac{1}{4 \sigma_Z})$, and $\mmax = m + \Wcovscale
  \log\frac{1}{\gamma}$
  for the $m$ in line~\ref{line:pm-m-setting} of
  $\privmean$ (though
  privacy does not depend on its value). We first achieve $2 \alpha \le
  \frac{\diffp}{4}$. Setting $\Wcovscale = \frac{16}{\diffp}$, it is
  sufficient that $\sigma_Z$ is large enough that $\frac{2 \sqrt{e} B
  (\mmax+1)}{n\sigma_Z} \le \frac{\diffp}{16}$, i.e., 
    \begin{align*}
      \sigma_Z \ge \frac{32 \sqrt{e} B (\mmax+1)}{n \diffp} = 
      \frac{32 \sqrt{e} B}{n \diffp} \left(m+1+ \frac{16\log \frac{1}{\gamma}}{\diffp}\right).
    \end{align*}
  For the $\delta$ privacy
  component, we wish to have $(1 + e^\alpha) \beta \le \delta$.  As we have
  guaranteed $\alpha \le \frac{\diffp}{4}$, taking $\gamma = \frac{\delta}{1
    + e^{\diffp/4}}$ and making sure $\sigma_Z$ is small enough that $n
  \exp(-\frac{1}{4 \sigma_Z}) \le \gamma = \frac{\delta}{1 + e^{\diffp/4}}$
  suffices. For this, it is evidently sufficient that
  $\frac{1}{\sigma_Z} \ge 4 \log \frac{n(1 + e^{\diffp/4})}{\delta}$,
  i.e., (substituting for $\sigma_Z$)
  \begin{equation*}
    n \ge
    \frac{128 \sqrt{e} B \log \frac{n(1 + e^{\diffp/4})}{\delta}}{\diffp}
    \left(m + 1 + \frac{16}{\diffp} \log \frac{1 + e^{\diffp/4}}{\delta}
    \right)
  \end{equation*}
  guarantees $\est{\Sigma}_{-i} \close{\diffp/4}{\delta}
  \est{\Sigma}'_{-i}$.
  
  \paragraph{Condition~\ref{item:internal-stability}.} In
  Lemma~\ref{lemma:clipr-internal-stability}, if the scale of the noise
  $\sigma_Z$ on $Z_j \simiid \laplace(\sigma_Z)$ satisfies $\exp(-\frac{1}{4
    \sigma_Z}) \le \gamma$, Condition~\ref{item:internal-stability} holds.
  The choice of $\sigma_Z$ to satisfy
  Condition~\ref{item:external-stability} above and
  the lower bound on $n$ are evidently sufficient.

  \paragraph{Condition~\ref{item:mean-sample}.} Lemma~\ref{lemma:mean-sample}
  guarantees that if $\msnznr_j \simiid \normal(0, \msnscale^2)$, $W \sim
  \laplace(\Wmeanscale)$, and $\partition \sim \uniform(\Partitions_{n,b})$,
  the call to $\meansafe$ in line~\ref{line:pm_ms-call} of $\privmean$
  gives
  $\msout(x, \genpsd) \close{\alpha + 1/\Wmeanscale}{\beta_1 +
    \beta_2} \msout(x', \genpsd)$,
  with $\Delta$, $\beta_2$ and $\msnscale$ as defined in the lemma.
  To achieve $\alpha + \frac{1}{\Wmeanscale} = \frac{\diffp}{4}$, take
  $\Wmeanscale = \frac{8}{\diffp}$ and choose $\alpha = \frac{\diffp}{8}$.
  To achieve $\beta_1 + \beta_2 \le \delta$, choose $\beta_1 = \frac{\delta}{2}$ and
  then recognize that $\beta_2 \leq \frac{\delta}{2}$ as long as
  $\gamma \le \frac{\delta}{6}$,
  $n^2 2^{1 - b} \le \frac{\delta}{6}$
  (or $b \ge \log_2 \frac{6n^2}{\delta} + 1$) and
  $\half \exp(-\frac{k/3 + 1}{\Wmeanscale}) \le \frac{\delta}{6}$
  (or $k \ge \frac{24}{\diffp} \log \frac{3}{\delta} - 3$).
  Thus, we arrive at
  \begin{equation*}
    \msnscale = \frac{8 \Delta}{\diffp}
    \sqrt{\log\frac{5}{2 \delta}}
    = \frac{20 \sqrt{B} b}{n \diffp}
    \exp\left(3 \mstopkscale \log \frac{12 n}{b \delta}\right)
  \end{equation*}
  for (any) $b \ge 2 \log \frac{6n}{\delta} + 1$ so long as
  $\frac{\diffp}{8} \le 1$. (Otherwise we may use the alternative
  value for $\msnscale$ preceding Lemma~\ref{lemma:mean-sample},
  which the $(\diffp, \delta)$-differential privacy
  guarantee of Lemma~\ref{lemma:normal-closeness-mean} justifies.)

  \paragraph{Condition~\ref{item:mean-covariance}.}
  The last condition to verify is that $\msout(x, \genpsd)
  \close{\diffp/4}{\delta} \msout(x, \genpsd')$ for close enough
  $\genpsd,\genpsd'$. For this, we use Lemma~\ref{lemma:mean-covariance}, which
  guarantees that $\msout(x, \genpsd) \close{\alpha_1 + \alpha_2}{\delta}
    \msout(x, \genpsd')$ for $\alpha_1 = \frac{6a}{n} \log \frac{2}{\delta}$,
  where we take $a = \frac{B \sqrt{e}}{1 - B \sqrt{e}/n}$ via
  Lemma~\ref{lemma:clipr-internal-stability}, and arbitrary $\alpha_2 > 0$.
  Set $\alpha_2 = \frac{\diffp}{8}$ and obtain 
  $\mstopkscale = \frac{8 k a}{n \diffp}$.
  When $n \geq \frac{48a}{\eps}\log \frac{2}{\delta}$, we have
  $\alpha_1 \leq \frac{\eps}{8}$, and so the desired privacy holds.

  Making appropriate substitutions gives
  that each of
  conditions~\ref{item:internal-stability}--\ref{item:mean-covariance} holds
  with parameters $(\diffp/4, \delta)$.
  Proposition~\ref{proposition:condition-composition} gives the theorem.
\end{proof}

\section{Accuracy analysis}\label{sec:accuracy}

The second important component of our analysis of $\privmean$ is its
accuracy. We provide two accuracy results: the first (Theorem~\ref{thm:pm-acc})
covers the case in which the data is sub-Gaussian, where we assume
the method has some knowledge of the sub-Gaussian parameter of the
sampling distribution. Of course,
it is unreasonable to assume that a given distribution is sub-Gaussian or
that we know its sub-Gaussian norms, and thus we extend $\privmean$
via a procedure that adapts to the actual scale of the data
in Section~\ref{sec:adm-analysis}. 

Throughout this section, we let $P$ be a distribution on $\R^d$ with mean
$\mu$ and covariance $\Sigma$, and we assume $X_i \simiid P$, $i=1, \ldots,
n$.  The classical (non-private) sample mean and covariance are $\samplemean
= \frac{1}{n}\sum_{i=1}^n X_i$ and $\samplecov = \frac{1}{n}
\sum_{i=1}^{n/2} (X_i - X_{n/2+i})(X_i - X_{n/2+i})^T$.
We assume throughout that $P$ enjoys certain concentration properties,
though we emphasize that our methods will be adaptive to the 
parameters we specify here.
\begin{assumption}[Sample concentration]
  \label{ass:P-conc}
  Let $c_1 \ge 64e$ and $\assfail \in (0, 1)$. For
  $X_i \simiid P$ with $\Ep[X] = \mu$ and $\cov(X) = \Sigma$,
  there exists $\assdiam < \infty$ such that the event
  \begin{equation*}
    \Es \defeq \Big\{
        \max_{i \in [n]} \norm{X_i - \mu}_\Sigma^2 \leq \assdiam^2/c_1
        \text{ and }
        \tfrac{1}{2} \Sigma \preceq \samplecov \preceq \tfrac{3}{2}\Sigma
      \Big\}
  \end{equation*}
  occurs with probability at least $1-\assfail$.
\end{assumption}

It is useful to give some context for the values of $\assdiam$ we
expect under various distributional assumptions.
Because $\Ep[\norm{X_i - \mu}_\Sigma^2] = d$, the constant $\assdiam^2$
typically scales at least as $d$.
We now give more detailed examples.
In each, we let $Z_i = \Sigma^{-1/2}(X_i - \mu)$ 
be the whitened data, defining the sample
covariance $\whitenedcov = \frac{1}{n}\sum_{i=1}^{n/2} (Z_i - Z_{n/2+i})(Z_i
- Z_{n/2+i})^T$.  Because $\norm{X_i - \mu}_{\Sigma} = \norm{Z_i}_2$ and
$\samplecov = \Sigma^{1/2}\whitenedcov\Sigma^{1/2}$, we have
the equivalence
\begin{equation*}
  \Es = \left\{\max_{i \in [n]} \norm{Z_i}_2^2 \leq \assdiam^2/c_1
  ~~ \mbox{and} ~~
  \opnorm{\whitenedcov - I} \leq \half \right\}.
\end{equation*}

\begin{example}[Sub-Gaussian random vectors]
  \label{example:sub-g-vecs}
  If for all $v$ satisfying $\ltwo{v} \le 1$ the scalar
  random variable $\<Z, v\>$ is $\tau^2$-sub-Gaussian,
  \begin{equation*}
    \assdiam^2 \le O(1)\tau^2 \left[d + \log \frac{n}{\beta}\right].
  \end{equation*}
  Indeed, a standard covering argument
  (see, e.g.,~\cite[Ch.~5]{Wainwright19} or~\cite[Ch.~4]{Vershynin19})
  gives that for all $t \ge 0$,
  $\P(\ltwo{Z} \ge t)
  \le 4^d \exp(-c t^2 / \tau^2)$, where $c > 0$ is a numerical constant.
  Replacing $t^2$ with $O(1)(d \tau^2 + \tau^2 t^2)$ gives that
  $\P(\ltwo{Z} \ge C \tau \sqrt{d + t^2}) \le \exp(-t^2)$,
  and for any $\gamma > 0$, setting
  $t^2 = \log \frac{n}{\gamma}$ yields that with
  probability at least $1 - \gamma$,
  \begin{equation*}
    \max_{i \le n} \ltwo{Z_i}^2 \le O(1)
    \tau^2 \left[d + \tau^2 \log \frac{n}{\gamma}\right].
  \end{equation*}

  To control the covariance, we use~\citet[Theorem 5.39]{Vershynin12},
  which gives that
  with probability at least $1 - 2 e^{-c t^2}$,
  $\opnorm{\whitenedcov - I} \le
  O(1) \tau^2 \max\{\sqrt{d / n} + t / \sqrt{n},
  d/n + t^2 / n\}$, so that (igoring the sub-Gaussian constant)
  for $n \gtrsim d$, setting $t^2 = O(1) \log \frac{1}{\gamma}$
  gives $\opnorm{\whitenedcov - I} \le \half$ with
  probability at least $1 - \gamma$.
  Set $\gamma = \beta/2$.
\end{example}

\begin{example}[General moment bounds]
  \label{example:general-moment}
  Suppose for some $p \ge 4$ we have $\Ep[\norm{X_i-\mu}_\Sigma^p]
    = \Ep[\norm{Z_i}_2^p] \le \tau^p d^{p/2}$,
  where necessarily $\tau \geq 1$.
  Then we can give two results: the first being that asymptotically $\assdiam =
  o(n^{1/p})$ and the second more quantitative.  For the first, we claim
  that $\max_{i \le n} \ltwo{Z_i} / n^{1/p} \cas 0$. To see this, note that
  for any $\eps > 0$,
  \begin{align*}
    \infty > \frac{1}{\eps^p} \Ep[\ltwo{Z_1}^p] = \int_{0}^\infty \P(\ltwo{Z_1}^p \geq \eps^p t) \, dt
    \geq \sum_{i = 1}^\infty \P(\ltwo{Z_i}^p \ge \eps^p i).
  \end{align*}
  By the Borel-Cantelli lemma, the event $\ltwo{Z_n} \ge \eps n^{1/p}$
  occurs only finitely often, and so the claim follows.
  Meanwhile, the strong law of large numbers guarantees that $\whitenedcov \cas I$.


  For more quantitative parameters,
  we first get by Markov's inequality that
  \begin{equation*}
    \P(\max_{i \in [n]} \norm{Z_i}_2 > t)
    \leq \frac{n\Ep[\norm{Z_1}_2^p]}{t^p}
    \leq \frac{n\tau^p d^{p/2}}{t^p},
  \end{equation*}
  so setting $M \asymp \tau \sqrt{d} n^{1/p} / \beta^{1/p}$, we have
  $\max_i \norm{Z_i}_2 \leq M/c_1$ with probability at least $1 - \beta$.
  To show concentration of the covariance matrix, we apply
  \citet[Theorem A.1 Part 2]{ChenGiTr12}, treating $p$ as a constant, obtaining
  \begin{align*}
    n\Ep\big[\opnorm{\whitenedcov - I}^{p/2}\big]^{2/p}
    &\lesssim \sqrt{n\log d}\sqrt{\Ep[\norm{Z}_2^4]}
      + (n^{2/p}\log d) \Ep[\norms{Z_1}_2^{p}]^{2/p} \\
    &\lesssim \max\{\sqrt{n\log d}, n^{2/p}\log d\}
      \Ep[\norms{Z_1}_2^{p}]^{2/p},
  \end{align*}
  and so by Markov's inequality
  \begin{equation*}
    \P(\opnorm{\whitenedcov - I} > \tfrac{1}{2})
    \lesssim \max\{n^{-p/4}\log^{p/4} d, n^{1-p/2}\log^{p/2} d\} \Ep[\norms{Z_1}_2^{p}]
    \lesssim \frac{\tau^p (d\log d)^{p/2}}{n^{p/2-1}},
  \end{equation*}
  which has bound $\beta$ when $n \gtrsim (\tau^2 d\log d)^{p/(p-2)}
  \beta^{-2/(p-2)}$.
\end{example}



\subsection{Accuracy of \texorpdfstring{$\privmean$}{PRIVMEAN}}
  \label{sec:pm-accuracy}

We give our promised accuracy guarantee
whenever Assumption~\ref{ass:P-conc} holds.
Though not strictly necessary, we state the theorem assuming that $\delta$ is not too small to allow
for a cleaner result. Throughout,
$c$ denotes a numerical constant whose value can change from line to line.
\begin{theorem}\label{thm:pm-acc}
  Let $\eps > 0$ and $e^{-d} \leq \delta \leq \frac{1}{n}$ be privacy parameters
  and let Assumption~\ref{ass:P-conc} hold.
  Let $B \geq \assdiam^2$ and suppose
  $n \geq \frac{c}{\diffp^2} B\log^2 \frac{1}{\delta}$.
  Let $\pmout = \privmean_{B, \eps, \delta}(X_{1:n})$.
  Then with probability at least $1 - (\assfail + 5\delta)$, $\pmout \neq \bot$ and
  \begin{equation*}
    \norm{\pmout - \samplemean}_\Sigma
    \leq\frac{c\sqrt{Bd}\log(\frac{1}{\delta})}{n\eps}.
  \end{equation*}
\end{theorem}

\begin{proof}
  We first show under the event $\Es$ that with probability at least $1-4\delta$
  over the randomness in $\privmean$, both $\clipr$ and $\meansafe$ prune no 
  observations, meaning the sets of
  removed indices $\removedinds = \emptyset$ in both procedures
  (so that Line~\ref{line:clipr-truncate} in $\clipr$ and
  Line~\ref{line:ms_m-check} in $\meansafe$ never fail), and thus 
  $\pmout = \samplemean + \samplecov^{1/2}\msnznr$.
  As $\samplecov \preceq \frac{3}{2} \Sigma$ on $\Es$, we have that
  $\norms{\samplecov^{1/2}\msnznr}^2_\Sigma \leq \frac{3}{2} \ltwos{\msnznr}^2$.
  The result then follows follows once we show that
  $\ltwos{\msnznr} \leq \frac{c\sqrt{Bd}\log(\frac{1}{\delta})}{n\eps}$
  with probability at least $1-\delta$ and take a union bound over 
  these events and $\Es$.

  Rearranging the condition in Line~\ref{line:clipr-truncate} of $\clipr$,
  the element $X_i - X_{n/2+i}$ is pruned in the first iteration only if
  \begin{align*}
    (\pmclnznormr)_i + (\pmclnznormr)_{n/2+1} 
    & > \log(B) - \log(\norm{X_i - X_{n/2 + i}}_{\samplecov}^2) \\
    & \stackrel{(\star)}{\geq} \log(c_1 B/8 M^2) \geq \log(c_1/8),
    \label{eq:cs-prune-cond}
  \end{align*}
  where ($\star$) holds for all $i \in [n/2]$ on event $\Es$ because
  \begin{equation}
    \norm{X_i - X_{n/2 + i}}_{\samplecov}^2
    \leq 2 \norm{X_i - X_{n/2 + i}}_{\Sigma}^2
    \leq 4\norm{X_i - \mu}_{\Sigma}^2 + 4\norm{X_{n/2 + i} - \mu}_{\Sigma}^2
    \leq 8\assdiam^2/c_1. \label{eqn:acc-norm-arg}
  \end{equation}
  As $c_1 \geq 64e$ by Assumption~\ref{ass:P-conc}, if
  $\norm{\pmclnznormr}_\infty \leq 1/2$ then $\clipr$ in
  line~\ref{line:clipr-truncate} prunes no entries, instead simply passing
  $\samplecov$ to $\meansafe$ so long as $\pmclnzthrr + m > 0$
  (see line~\ref{line:clipr-check}).  Recall that
  $(\pmclnznormr)_j \simiid \laplace(\sigma_Z)$ for $j=1, \dots, n/2+1$ and
  $\sigma_Z = \frac{32 \sqrt{e} B (\mmax+1)}{n \diffp} \leq \frac{c B\log
    \frac{1}{\delta}}{n \eps^2}$, so by taking a union bound over the
  entries, we have with probability at least $1-\delta$ that
  \begin{equation*}
    \norm{\pmclnznormr}_\infty
    \leq \frac{cB\log \frac{1}{\delta}}{n \eps^2}
      \log\left(\frac{n/2+1}{\delta}\right)
    \leq \frac{1}{2},
  \end{equation*}
  where the last inequality is by the assumptions that
  $n \geq \frac{c B\log^2(\frac{1}{\delta})}{\eps^2}$ and
  $\delta \leq \frac{1}{n}$.
  Also recall that $\pmclnzthrr \sim \laplace(\frac{16}{\diffp})$ and
  $m = \frac{16}{\diffp}\log \frac{1}{\delta}$,
  so $\pmclnzthrr + m > 0$ with probability at least
  $1-\frac{\delta}{2}$.
  
  Continuing to the next phase of $\privmean$,
  $\meansafe$ with input $A = \samplecov$ prunes the indices $\partitionset_j$
  only if
  \begin{equation*}
    \wt{D}_j
    = D_j + (\pmmsnztopiir)_j
    > \log(\sqrt{B}/4).
  \end{equation*}
  By the same argument we used to
  obtain inequality~\eqref{eqn:acc-norm-arg}, on $\Es$ we have
  for all $j \in [n/b]$ that
  \begin{equation*}
    D_j = \log(\diam_{\samplecov}(X_{\partitionset_j}))
    \leq \log(\sqrt{8\assdiam^2/c_1}),
  \end{equation*}
  and so if $\norm{\pmmsnztopiir}_\infty \leq 1/2$ then
  \begin{equation*}
    \wt{D}_j
    \leq \log(\sqrt{8\assdiam^2/c_1}) + \frac{1}{2}
    \leq \log(\sqrt{B}/4),
  \end{equation*}
  where the last inequality follows from $c_1 \geq 64e$ and $B \geq \assdiam^2$.
  Thus, $\meansafe$ prunes no entries,
  and $\msout = \samplemean + \samplecov^{1/2} \msnznr$ so long as
  $\pmmsnzthrr + \frac{2k}{3} > 0$ (see Line~\ref{line:ms_thres-comp}).
  Recall that $(\pmmsnztopiir)_j \simiid \Lap(\mstopkscale)$
  for $j = 1,\dots,n/b$ and $\mstopkscale = \frac{8 k}{n \diffp}
  \frac{B \sqrt{e}}{1 - B \sqrt{e} / n}
  \leq \frac{cB\log(\frac{1}{\delta})}{n\eps^2}$. Another union bound
  gives that with probability at least $1-\delta$,
  \begin{equation*}
    \norm{\pmclnznormr}_\infty
    \leq \frac{cB\log(\frac{1}{\delta})}{n\eps^2}\log\frac{n}{b \delta}
    \leq \frac{1}{2},
  \end{equation*}
  where the last inequality follows from the  assumption
  $n \geq \frac{c B\log^2\frac{1}{\delta}}{\eps^2}$ and
  $\delta \leq \frac{1}{n}$.
  Also, $\pmmsnzthrr \sim \laplace(\frac{8}{\eps})$ and
  $\frac{2k}{3} = \frac{16}{\eps}\log \frac{3}{\delta} - 2$, so
  $\pmmsnzthrr-k > \frac{8}{\eps}\log \frac{3}{\delta} - 2 \geq 0$
  with probability at least $1-\delta$.

  Therefore, $\privmean$ returns $\samplemean + \samplecov^{1/2}\msnznr$
  with probability at least $1-4\delta$ on the event $\Es$.  Recall that
  $\msnznr \sim \normal(0, \msnscale^2 I)$ with $\msnscale =
  \frac{20b\sqrt{B}}{n \diffp} \exp(3 \mstopkscale \log \frac{12n}{b
    \delta})$, and because $\mstopkscale \leq
  \frac{cB\log(\frac{1}{\delta})}{n\eps^2}$, $\delta \leq \frac{1}{n}$, and
  $n \geq \frac{c B\log^2(\frac{1}{\delta})}{\eps^2}$, we have that
  $\msnscale \leq \frac{c\sqrt{B}\log(\frac{1}{\delta})}{n\diffp}$.
  Classical tail bounds on the $\chi^2$-distribution~\cite[Lemma
    1]{LaurentMa01} give with probability at least $1-\delta$ that
  \begin{equation*}
    \ltwobig{\msnznr}^2
    \leq \msnscale^2 \left[d + 2\sqrt{d\log \tfrac{1}{\delta}}
      +2\log\tfrac{1}{\delta}\right]
    \leq \frac{c B d \log^2\frac{1}{\delta}}{n^2 \eps^2},
  \end{equation*}
  where the last inequality follows from the bound on $\msnscale$ and
  assumption that $e^{-d} \leq \delta$.
\end{proof}

\subsection{Adapting to heavy-tailed data}\label{sec:adm-analysis}

In practice, we may not have \emph{a priori} knowledge of the concentration
properties of the data.
Given the necessarily slowed rates of convergence for private estimators
of means of random variables with only $p$ moments~\cite{BarberDu14a},
it is essential to be adaptive to the actual scale (and number of moments)
of the problem. We therefore develop
$\adamean$, which automatically tunes the
threshold parameter $B$ by repeatedly calling $\privmean$ and doubling $B$
until $\pmout \neq \bot$. The key is that upon termination of $\adamean$, 
the effective $B$ is at most twice the
realized scale $O(1) \max_{i \le n} \norm{Z_i}^2$ of the random variables.
To ensure privacy irrespective of the number of
calls to $\privmean$, with each successive call $\adamean$ progressively
decreases the privacy budget allocated to $\privmean$; in particular, as
$\sum_{t=1}^{\infty} 1/t^2 = \pi^2/6$, via basic composition we can bound
the total privacy loss of $\adamean$ by a factor $\frac{\pi^2}{6}$ over the
privacy loss of $\privmean$.  Aside from an extra factor polylogarithmic in
$B/d$, $\adamean$ matches the accuracy of $\privmean$, as we show presently.

\begin{algorithm}
  \label{algorithm:adamean}
  \DontPrintSemicolon
  \SetKwInOut{Input}{Input}
  \SetKwInOut{Output}{Output}
  \SetKwInOut{Params}{Params}
  \Input{ data $x_{1:n}$}
  \Params{ privacy budget $(\diffp,\delta)$}
  \Output{ mean estimate $\msout$ }

  \For{$t = 1, 2, \dots$}{
    $\msout_t \leftarrow \privmean_{d 2^{t-1},(\eps/t^2,\delta/t^2)}(x)$\;
    \If{$\msout_t \neq \perp$}{
      \Return{$\msout_t$}
    }
  }
  \caption{Fully Adaptive Private Mean Estimation ($\adamean$)}
\end{algorithm}

\begin{theorem}[Accuracy of $\adamean$]
  \label{theorem:adamean}
  Let $\eps > 0$ and $e^{-d} \leq \delta \leq \frac{1}{n}$ be privacy parameters 
  and let event $\Es$ hold.
  Let $s = \max\{1, \log_2\frac{4\assdiam^2}{d}\}$ and assume
  $n \geq \frac{c \max \{d, \assdiam^2\} s^2}{
    \diffp^2} \log^2\frac{s^2}{\delta}$.
  Let
  $\pmout = \adamean_{\eps, \delta}(X_{1:n})$.
  Then with
  probability at least $1 - (5 + \pi^2/3)\delta$,
  \begin{equation}
    \norm{\pmout - \samplemean}_\Sigma
    \leq \frac{cs^2\log(\frac{s^2}{\delta})\max\{\assdiam\sqrt{d},
      \assdiam \log(\frac{s^2}{\delta}), d\}}{n\diffp}.
    \label{eqn:adamean-claim}
  \end{equation}
\end{theorem}
\begin{proof}
  Let $\tstar$ be the smallest positive integer such that
  $\assdiam^2 \leq d2^{\tstar-1}$
  and let $\tstop$ be the iteration when $\adamean$ terminates
  (which may be infinite).
  Note that $\tstar \leq s$.
  The proof comes in two parts: on the event $\Es$, we first show that either
  $\tstop > \tstar$ or $\pmout$ satisfies the claim~\eqref{eqn:adamean-claim}
  with probability at least $1-\pi^2\delta/6$;
  secondly, we show $\adamean$ terminates 
  with $\tstop \leq \tstar$ with probability at least $1-5\delta$.
  The result then follows via a union bound.
  
  We carry out the first part with the help of the following lemma.
  \begin{lemma}\label{lemma:pm-acc-small-B}
    Let $\diffp > 0$ and $e^{-d} \leq \delta \leq \frac{1}{n}$ be privacy
    parameters and let $B \geq 0$.
    Suppose the event $\Es$ holds and let
    $\pmout = \privmean_{B, (\diffp, \delta)}(X_{1:n})$.
    Then with probability at least $1 - 2\delta$
    over the randomness of $\privmean$, $\pmout = \bot$ or
    \begin{equation*}
      \norm{\pmout - \samplemean}_\Sigma
      \leq \frac{c\log \frac{1}{\delta}\max\{\assdiam \log\frac{1}{\delta},
        \sqrt{Bd}\}}{n\diffp}.
    \end{equation*}
  \end{lemma}
  \begin{proof}
    Suppose $\pmout \neq \bot$ as otherwise the claim is trivial.
    Let $\clout$ denote the covariance estimate of
    $\clipr$ (that is, $\Sigma_T$ at the final
    iteration of $\clipr$), and let $\msmean$ denote the empirical mean of 
    the observations not pruned by $\meansafe$
    so that $\pmout = \msmean + \clout^{1/2} \msnznr$.
    Then by the condition for returning $\bot$ in Line~\ref{line:ms_bot-output} of $\meansafe$, $\meansafe$ prunes at most $b(\frac{2k}{3} + \pmmsnzthrr)$
    points and so
    \begin{align*}
      \norm{\msmean - \samplemean}_\Sigma
      & \leq \frac{b(\frac{2k}{3} + \pmmsnzthrr) \max_i \norm{X_i - \msmean}_\Sigma}{n} \\
      & \leq \frac{b(\frac{2k}{3} + \pmmsnzthrr) \( \max_i \norm{X_i - \mu}_\Sigma + \norm{\mu - \msmean}_\Sigma \)}{n} \\
      & \leq \frac{2b(\frac{2k}{3} + \pmmsnzthrr) \max_i \norm{X_i - \mu}_\Sigma}{n} 
      \stackrel{(\star)}{\leq} \frac{2 b(\frac{2k}{3} + \pmmsnzthrr) \assdiam}{\sqrt{c_1} n},
    \end{align*}
    with ($\star$) following directly from $\Es$.
    Recalling that $k = \frac{24}{\diffp}\log\frac{3}{\delta} - 3$ and
    $W \sim \laplace(\frac{8}{\diffp})$, it follows that
    $\frac{2k}{3} + W < \frac{24}{\diffp}\log\frac{3}{\delta}$ with probability
    at least $1-\frac{\delta}{6}$.
    Recalling also that $b = 1 + \log_2\frac{6n^2}{\delta}$,
    it follows that on $\Es$,
    \begin{equation*}
      \norm{\msmean - \samplemean}_\Sigma
      \leq \frac{2(1 + \log_2\frac{6n^2}{\delta})(\frac{24}{\diffp}\log\frac{3}{\delta}) \assdiam}{\sqrt{c_1}n}
      \leq \frac{c \assdiam \log^2 \frac{1}{\delta}}{n\diffp}
    \end{equation*}
    with probability at least $1-\frac{\delta}{6}$.

    Meanwhile, observe $\Es$ implies $\clout \preceq \samplecov \preceq 2\Sigma$
    as pruning entries (line~\ref{line:clipr-truncate}) in $\clipr$ only
    shrinks its covariance estimate.  Thus, $\norms{\clout^{1/2}
      \msnznr}_{\Sigma} \leq \sqrt{2}\norms{\clout^{1/2} \msnznr}_{\clout} =
    \sqrt{2}\ltwos{\msnznr}$.  From the same argument as in
    Theorem~\ref{thm:pm-acc},
    \begin{equation*}
      \ltwobig{\msnznr}
      \leq \frac{c\sqrt{Bd}\log(\frac{1}{\delta})}{n\eps}
    \end{equation*}
    with probability at least $1-\delta$.

    The preceding two displays together imply
    Lemma~\ref{lemma:pm-acc-small-B} after taking a union bound.
  \end{proof}

  Applying Lemma~\ref{lemma:pm-acc-small-B} with the mapping
  $B = d 2^{t-1}$, $\diffp \mapsto
  \diffp/t^2$ and $\delta \mapsto \delta/t^2$, we have
  for any $1 \leq t \leq
  \tstar$ that under the event $\Es$, with probability at least
  $1-2\delta/t^2$, either $\pmout_t = \bot$ or
  \begin{equation*}
    \norm{\pmout_t - \samplemean}_\Sigma
    \leq \frac{ct^2\log(t^2/\delta)\max\{\assdiam \log(t^2/\delta),
      d2^{(t-1)/2}\}}{n\diffp},
  \end{equation*}
  where the latter case $\pmout_t$ satisfies Eq.~\eqref{eqn:adamean-claim} as
  $t \leq \tstar \leq s$.
  Then via a union bound this same event holds
  simultaneously for all $1 \leq t \leq \tstar$ with probability at least
  $1-\pi^2 \delta/3$, and thus either $\tstop > \tstar$ or $\adamean$ terminates
  and $\pmout$ satisfies the claim~\eqref{eqn:adamean-claim}.

  Proceeding to the second part of the proof, recall that
  $d2^{\tstar-1} \geq \assdiam^2$ and so by applying Theorem~\ref{thm:pm-acc}
  with $B = d2^{\tstar-1}$, it follows under $\Es$ that
  $\pmout_{\tstar} \neq \bot$, and thus
  $\adamean$ terminates after $\tstop \leq \tstar$ iterations,
  with probability at least $1-5\delta/(\tstar)^2 \geq 1-5\delta$.
  The claim~\eqref{eqn:adamean-claim} follows.
\end{proof}

\begin{example}[Example~\ref{example:sub-g-vecs} continued]
  In this case, $\Sigma^{-1/2}X_i$ is $\tau^2$-sub-Gaussian,
  so 
  $M^2 \lesssim \tau^2 (d + \log\frac{n}{\beta})$
  in Assumption~\ref{ass:P-conc}.
  Thus, the sample mean concentrates as
  $\norm{\samplemean - \mu}_\Sigma
  \lesssim \tau \sqrt{(d + \log(1/\delta)) / n}$ with
  probability at least $1 - \delta$, and
  assuming $\delta \geq e^{-d}$, Theorem~\ref{theorem:adamean} then implies
  with probability at least $1 - O(\delta)$ over
  $\pmout = \adamean_{\eps, \delta}(X_{1:n})$ that (ignoring polylogarithmic factors in $n$)
  \begin{equation*}
    \norm{\pmout - \mu}_\Sigma
    = \wt{O}\left(\tau \sqrt{\frac{d}{n}} + \frac{\tau d \log \frac{1}{\delta}}{n\eps}\right).
  \end{equation*}
  This rate is, up to a factor of $\log \frac{1}{\delta}$ and
  polylogarithmic factors in $n$, minimax-optimal for the sub-Gaussian
  setting (see~\cite{SteinkeUl17} or \cite[Lemma 6.7]{KamathLiSiUl18} for a
  lower bound on Gaussian mean estimation with known covariance matrix).
\end{example}

\begin{example}[Example~\ref{example:general-moment}, continued]
  \label{example:general-moment-2}
  Recall here that $\Ep[\norm{X_i - \mu}_{\Sigma}^p] \leq \tau^p d^{p/2}$
  for $p \geq 4$ and $\tau \geq 1$.
  By Theorem~\ref{theorem:adamean},
  with probability at least
  $1 - 5 \delta$ over
  $\pmout = \adamean_{\eps, \delta}(X_{1:n})$, we have
  \begin{equation*}
    \norm{\pmout - \mu}_\Sigma
    \le \norm{\wb{X}_n - \mu}_\Sigma
    + \wt{O}\left(\frac{\max_i \norm{X_i - \mu}_\Sigma
      \sqrt{d} \log \frac{1}{\delta}}{n \diffp}\right)
  \end{equation*}
  so long as the empirical covariance
  satisfies $\half \Sigma \preceq \wb{\Sigma}_n \preceq \frac{3}{2}\Sigma$.
  As this occurs with constant probability and
  $\norms{\wb{X}_n - \mu}_\Sigma \lesssim \sqrt{d/n}$ with constant probability,
  we substitute the bounds on $\max_i \norm{X_i - \mu}_\Sigma$
  from Example~\ref{example:general-moment} to obtain that
  with (any) constant probability,
  \begin{equation*}
    \norm{\pmout - \mu}_\Sigma
    = \wt{O}\left(\sqrt{\frac{d}{n}}
    + \frac{\tau d \log \frac{1}{\delta}}{n^{1 - 1/p} \diffp}\right).
  \end{equation*}
  By combining minimax lower bounds of \citet[Proposition 4]{BarberDu14a}
  and~\citet{SteinkeUl17}, the best known minimax lower bound is that with
  constant probability,
  \begin{equation*}
    \norm{\pmout - \mu}_\Sigma \gtrsim
    \sqrt{\frac{d}{n}} + 
    \tau \frac{d^\frac{2p - 1}{2p} \log^\frac{p - 1}{2 p} \frac{1}{\delta}}{
      (n \diffp)^{1 - 1/p}}.
  \end{equation*}
  The adaptive method thus achieves optimal scaling in
  $n$, but it may be loose in $\eps$ and off by a factor of
  $d^{1/2p}$ in dimension dependence.
\end{example}



\section{Proofs for stable covariance estimation}
\label{sec:covariance}

In this section, we provide the proofs of
Lemmas~\ref{lemma:clipr-internal-stability}
and~\ref{lemma:clipr-loo-stability}, though we begin with a collection of
preliminary results that allow us to actually prove the main two lemmas. In
the proofs, we refer to each execution of the while loop beginning in
Line~\ref{line:clipr-while} of $\clipr$ as an \textit{iteration} of $\clipr$
and use the transcript $\Transcript$ as a convenient means for tracking the
full execution of $\clipr$ through all iterations.

\subsection{Properties of \texorpdfstring{$\clipr$}{COVSAFE} }

We first formalize deterministic properties about the execution of $\clipr$,
giving conditions under which outputs of $\clipr$ are quite
stable. In the sequel, we use these to give
sets to which the noise variables $\clnznormr$ and $\clnzthrr$
belong with high probability, guaranteeing stability.
Recall the notation~\eqref{eqn:all-instantiated-sigma}
that $\est{\Sigma}(x, z, w)$ is the output of $\clipr$ on input
sample $x$ and noise $z \in \R^{n/2 + 1}, w \in \R$,
with transcript $\Transcript = ([\removedinds_t]_{t \le T},
[\Sigma_t]_{t \le T}, T)$ depending implicitly on $(x, z, w)$,
and $\est{\Sigma}_{-i}(x,z,w)$ is the corresponding leave-one-out covariance.
We shorthand $\est{\Sigma} = \est{\Sigma}(x,z,w)$ and $\est{\Sigma}_{-i} = \est{\Sigma}_{-i}(x,z,w)$
and take $\clx = x_{1:n/2} - x_{n/2+1:n}$ as in
Line~\ref{line:x-transform} of $\clipr$.

Lemma~\ref{lemma:cov-prune-iteration}
gives necessary and sufficient
conditions for pruning $\clx_i$ in iteration $t+1$ of $\clipr$,
i.e., $i \in \removedinds_{t + 1}$, and
Lemma~\ref{lemma:cov-prune-ever}
gives similar conditions for ever pruning $\clx_i$ (that is,
whether $i \in \removedinds_T$).

\begin{lemma}\label{lemma:cov-prune-iteration}
  Index $i \in \removedinds_{t+1}$ if and only if 
  $\log \(\norm{\clx_i}_{\Sigma_{t}}^2\) + \clnznormf_i + \clnznormf_{n/2+1} > \log(B)$.
\end{lemma}
\begin{proof}
  The ``if'' direction is immediate from the condition for adding an
  element to $\removedinds_{t+1}$ (see line~\ref{line:clipr-truncate} of
  $\clipr$).  For the other direction, if $i \in \removedinds_{t+1}$ then
  (again from the same condition) we must have for some $s \leq t$ that
  \begin{align*}
    \log \(\norm{\clx_i}_{\Sigma_{s}}^2\) + \clnznormf_i + \clnznormf_{n/2+1} > \log(B).
  \end{align*}
  Because $s \le t$, we have $\removedinds_s \subset \removedinds_t$
  and therefore $\Sigma_s \succeq \Sigma_t$, this in turn implies
  $\log \(\norm{\clx_i}_{\Sigma_{t}}^2\) + \clnznormf_i + \clnznormf_{n/2+1} > \log(B)$.
\end{proof}

\begin{lemma}\label{lemma:cov-prune-ever}
  Index $i \notin \rset{\removedinds}_{T}$ if and only if
  $\log (\norm{\clx_i}_{\Sigma_{T}}^2) + \clnznormf_i + \clnznormf_{n/2+1} \leq \log(B)$.
\end{lemma}
\begin{proof}
  Observe that $\Sigma_{T-1} = \Sigma_{T}$ because the inner while loop of
  $\clipr$ terminates only if the algorithm prunes no observations in the
  previous iteration (see line~\ref{line:clipr-convergence} of $\clipr$).
  Then the claim follows by applying Lemma~\ref{lemma:cov-prune-iteration}
  with $t=T-1$.
\end{proof}

Finally, we may completely characterize $\est{\Sigma}_{-i}$ via
the removed indices $\removedinds_{T,-i}$ and the threshold $m + \clnzthrf$,
as prescribed by the lemma below.
\begin{lemma}
  \label{lemma:sigma-i-via-cardinality}
  $\est{\Sigma}_{-i} = \frac{1}{n}\sum_{j \notin \removedinds_{T,-i} \cup \{i\}} \clx_j
  \clx_j^T$ if and only if $\cardp{\rset{\removedinds}_{T}} \leq m + \clnzthrf$.
\end{lemma}
\begin{proof}
  The claim follows immediately from the return condition in
  Line~\ref{line:clipr-check} of $\clipr$, as $\est{\Sigma} \neq \bot$
  implies $\est{\Sigma}_{-i} = \est{\Sigma} - \frac{1}{n}\indic{i \notin
    \removedinds_{T}} \clx_i \clx_i^T$, where $\est{\Sigma} = \Sigma_T =
  \frac{1}{n}\sum_{j \notin \removedinds_T} \clx_j \clx_j^T$ by definition.
\end{proof}



\subsection{Proof of Lemma~\ref{lemma:clipr-internal-stability}}
\label{sec:proof-clipr-internal-stability}

We shorthand $\est{\Sigma} = \est{\Sigma}(x, Z, w)$ and $\est{\Sigma}_{-i} = \est{\Sigma}_{-i}(x,Z,w)$
throughout the proof.
Assume that $\est{\Sigma} \neq \bot$, as otherwise the result is trivial,
and recall
$\clnznormr_j \simiid \laplace(\sigma)$.
Observe if $i \in \removedinds_T$ we have $\clout = \clout_{-i}$ by
definition; thus, we need only consider $i \not\in \removedinds_T$.
Proceeding, Lemma~\ref{lemma:cov-prune-ever} gives
\begin{align*}
  \log \(\norm{\clx_i}_{\clout}^2 \) + \clnznormr_i + \clnznormr_{n/2+1} \leq \log(B)
\end{align*}
for all $i \notin \removedinds_T$, from which it follows that
$\norm{\clx_i}_{\clout}^2 \leq B\sqrt{e}$ whenever $\clnznormr_i +
\clnznormr_{n/2+1} \geq -1/2$. We now use that
$\norm{\clx_i}_{\clout}^2
\leq B\sqrt{e}$ implies
that $\dpsd(\est{\Sigma}, \est{\Sigma}_{-i})$ is small, which
follows from the following linear algebraic lemma.

\begin{lemma}
  \label{lemma:cov-nuc-stability}
  Let $A \in \R^{d\times d}$ be positive semi-definite and $a \in \R^d$ satisfy
  $\norm{a}_A^2 < 1$.
  Then
  \begin{equation*}
    \dpsd(A, A - aa^T) \le \frac{1}{1 - \norm{a}_A^2} \norm{a}_A^2.
  \end{equation*}
\end{lemma}
\begin{proof}
  Define $C = A - aa^T$ for shorthand.
  We first establish that $\colspace(C) = \colspace(A)$.
  Because $\norm{a}_A^2$ is finite, it follows that $a \in \colspace(A)$
  and so $\colspace(C) \subset \colspace(A)$.
  On the other hand, by expanding $C$ we have
  \begin{equation}
    C = A^{1/2}(I - A^{\dag/2} aa^T A^{\dag/2}) A^{1/2}
    \succeq (1-\norm{a}_A^2) A,
    \label{eqn:C-lb}
  \end{equation}
  thus implying that $\colspace(C) = \colspace(A)$.
  
  We also have from \eqref{eqn:C-lb} that
  $C^\dag \preceq \frac{1}{1-\norm{a}_A^2} A^\dag$, and so
  \begin{equation*}
    \nucnormbig{C^{\dag/2} (A - C) C^{\dag/2}}
    = \nucnormbig{C^{\dag/2} a a^T C^{\dag/2}}
    = \norm{a}_C^2 \leq \frac{\norm{a}_A^2}{1 - \norm{a}_A^2}.
  \end{equation*}
  A parallel calculation
  yields $\nucnorms{A^{\dag/2} (C - A) A^{\dag/2}} =
  \norm{a}_A^2$, proving the claim.
\end{proof}

Lemma~\ref{lemma:cov-nuc-stability} immediately shows that
$\est{\Sigma}_{-i} = \est{\Sigma} - \frac{1}{n} \clx_i \clx_i^T
\indic{i \in \removedinds_T}$ satisfies
\begin{align*}
  \dpsd(\est{\Sigma}, \est{\Sigma}_{-i})
  \le \frac{1}{1 - B \sqrt{e} / n}
  \frac{B \sqrt{e}}{n}
\end{align*}
whenever $\clnznormr_i + \clnznormr_{n/2 + 1} \ge -\half$.
To show that this occurs with high probability, we use the following
result,
which follows from the observation that if $c \ge 0$, then for
any independent variables $X, Y$ we have
$\P(X + Y > c) \le \P(X > c/2) + \P(Y > c/2)$ by a union bound:


\begin{observation}
  \label{observation:laplace-sum-tail}
  Let $X, Y \simiid \laplace(\sigma)$ and $c \ge 0$. Then $\P(X + Y > c) \le
  \exp(-\frac{c}{2\sigma})$.
\end{observation}
\noindent
We see that
$\P(Z_i + Z_{n/2 + 1} < -\half) \le \exp(-\frac{1}{4 \sigma})$ as claimed.

\subsection{Proof of Lemma~\ref{lemma:clipr-loo-stability}}
\label{sec:proof-clipr-loo-stability}

The proof of Lemma~\ref{lemma:clipr-loo-stability} comes in four steps.  The
crux of the proof is a coupling argument where, via the running assumption 
$\clx_{-i} = \clx_{-i}'$, we equate the execution of
$\clipr$ on $x$ to that on $x'$ by perturbing $\clnznormr$ in a careful
way that changes the distribution of $\clnznormr$ little.
Step one in this approach, which we provide in
Lemma~\ref{lemma:deterministic-removal-relations}, is a deterministic lemma
relating the collections $\removedinds$ and $\removedinds'$ of indices
$\clipr$ removes on adjacent inputs $x$ and $x'$ via the noise input values
$\clnznormf$. In the second and third steps, which consist
of Lemmas~\ref{lemma:clipr-set-probabilities}
and~\ref{lemma:clipr-perp} respectively, we construct a map
$\pi : \R^{n/2+1} \to \R^{n/2+1}$ such that $\clnznormr$
and $\pi(\clnznormr)$ have similar distributions and for
which $\est{\Sigma}_{-i}(x, \clnznormf,\clnzthrr)$
and $\est{\Sigma}_{-i}(x', \pi(\clnznormf),\clnzthrr)$
(recall the definition~\eqref{eqn:instantiated-sigma-minus-i}) likewise
have similar distributions for all $\clnznormf$,
where we use the randomness in $\clnzthrr$ for the latter distributional
approximation.  Lemma~\ref{lemma:clipr-set-probabilities} relates the
distributions of the removed indices $\removedinds_{T,-i}$, while
Lemma~\ref{lemma:clipr-perp} relates the probabilities that
$\clipr$ aborts and returns $\bot$. In
Sec.~\ref{sec:proof-finally-loo-stability},
we finally synthesize the intermediate lemmas to give the
proof of Lemma~\ref{lemma:clipr-loo-stability}.


Our first step is the deterministic lemma
relating the collections of removed indices.
\begin{lemma}
  \label{lemma:deterministic-removal-relations}
  Let $\clnznormf,\clnznormf' \in \R^{n/2+1}$ and $\clnzthrf \in \R$,
  and let
  \begin{align*}
    \est{\Sigma},\([\removedinds_t]_{t=0}^T,[\Sigma_t]_{t=0}^T,T\) & := \clipr_{\clparams}(x;\clnznormf,\clnzthrf), \\
    \est{\Sigma}',\([\removedinds_t']_{t=0}^{T'},[\Sigma_t']_{t=0}^{T'},T'\) & := \clipr_{\clparams}(x';\clnznormf',\clnzthrf).
  \end{align*}
  Assume $\clx_i' = 0$. The following
  hold.
  \begin{enumerate}[label=(\alph*),labelindent=0pt]
  \item \label{item:prime-superset-a} If $\clnznormf_j' + \clnznormf_{n/2+1}'
    \geq \clnznormf_j + \clnznormf_{n/2+1}$ for all $j \in
    \removedinds_{T,-i}$, then $\removedinds_{T,-i} \subset
    \removedinds_{T',-i}'$.
  \item \label{item:prime-superset-b}
    If $\clnznormf_j' + \clnznormf_{n/2+1}' \geq \clnznormf_j +
      \clnznormf_{n/2+1}$ for all $j \notin \removedinds_{T',-i}'$, then
      $\removedinds_{T,-i} \subset \removedinds_{T',-i}'$.
  \end{enumerate}
  Additionally assume that
  $n \geq 2B\sqrt{e}$ and $\clnznormf_i + \clnznormf_{n/2+1} \geq -1/2$. Then
  the following hold.
  \begin{enumerate}[label=(\alph*),labelindent=0pt]
    \setcounter{enumi}{2}
  \item \label{item:prime-subset-c}
    If $\clnznormf_j' + \clnznormf_{n/2+1}' \leq \clnznormf_j +
    \clnznormf_{n/2+1} - 2B\sqrt{e}/n$ for all $j \in
    \removedinds_{T',-i}'$, then $\removedinds_{T',-i}' \subset
    \removedinds_{T,-i}$.
  \item \label{item:prime-subset-d} If $\clnznormf_j' + \clnznormf_{n/2+1}'
    \leq \clnznormf_j + \clnznormf_{n/2+1} - 2B\sqrt{e}/n$ for all $j \notin
    \removedinds_{T,-i}$, then $\removedinds_{T',-i}' \subset
    \removedinds_{T,-i}$.
  \end{enumerate}
\end{lemma}
\begin{proof}
  We prove each claim by induction over $t \in \N$.
  
  Proceeding with the first claim~\ref{item:prime-superset-a}, observe
  trivially that $\removedinds_{0,-i} = \emptyset \subset
  \removedinds_{T',-i}'$.  Now suppose for the sake of induction that
  $\removedinds_{t,-i} \subset \removedinds_{{T}',-i}'$.  If $t = T$ then
  there is nothing to show, so we take $t < T$.  Our assumption that
  $\clx_{-i} = \clx_{-i}'$ and $\clx_i' = 0$ implies $\Sigma_t \succeq
  \Sigma_{T'}'$.  Thus, for all $j \in \removedinds_{t,-i} \subset
  \removedinds_{T,-i}$ we have
  \begin{align*}
    \log(B)
    \stackrel{(i)}{<} \log \(\norm{\clx_j}_{\Sigma_t}^2 \)
      + \clnznormf_j + \clnznormf_{n/2+1}
    & \stackrel{(ii)}{\leq} \log \(\norm{\clx_j}_{\Sigma_{T'}'}^2 \)
      + \clnznormf_j + \clnznormf_{n/2+1} \nonumber \\
    & \stackrel{(iii)}{\leq} \log \(\norm{\clx_j}_{\Sigma_{T'}'}^2 \)
      + \clnznormf_j' + \clnznormf_{n/2+1}',
  \end{align*}
  where inequality $(i)$ follows from
  Lemma~\ref{lemma:cov-prune-iteration} (applied with noise $\clnznormf$),
  inequality $(ii)$ because $\Sigma_t \succeq \Sigma_{T'}'$, and
  inequality $(iii)$ is by assumption in case~\ref{item:prime-superset-a}.
  Lemma~\ref{lemma:cov-prune-ever} (with data $x'$ and noise
  $\clnznormf'$) then gives that $j \in \removedinds'_{T', -i}$ if and only if
  $\log B < \log(\norm{\clx_j}_{\Sigma_{T'}'}^2) + \clnznormf_j' +
  \clnznormf'_{n/2 + 1}$, and as $j \in \removedinds_{t,-i}$ was arbitrary we
  have $\removedinds_{t+1,-i} \subset \removedinds_{T',-i}'$.
  This completes the inductive step and so
  $\removedinds_{t,-i} \subset \removedinds_{T',-i}'$ for all $t
  \leq T$ and the first claim holds.

  The proof of claim~\ref{item:prime-superset-b} relies on a similar
  inductive argument as that for the first
  claim~\ref{item:prime-superset-a}. Equivalent to the inclusion
  $\removedinds_{T,-i} \subset \removedinds'_{T',-i}$ is that if $j \not \in
  \removedinds'_{T', -i}$, then $j \not \in \removedinds_{T,-i}$. Consider
  $j \not\in \removedinds'_{T',-i}$, and begin with the inductive assumption
  that $\removedinds_{t, -i} \subset \removedinds'_{T',-i}$; it suffices to
  show that $j \not \in \removedinds_{t + 1, -i}$. Because $\Sigma_t \succeq
  \Sigma'_{T'}$ by construction of $\Sigma_t$, we obtain
  \begin{align*}
    \log \(\norm{\clx_j}_{\Sigma_t}^2 \) + \clnznormf_j + \clnznormf_{n/2+1}
    & \le
    \log \(\norm{\clx_j}_{\Sigma_{T'}'}^2 \)
    + \clnznormf_j + \clnznormf_{n/2+1} \\
    & \stackrel{(i)}{\leq}
    \log \(\norm{\clx_j}_{\Sigma_{T'}'}^2\) + \clnznormf_j' + \clnznormf_{n/2+1}'
    \leq \log(B),
  \end{align*}
  where step~$(i)$ is by the assumption
  that $\clnznormf_j + \clnznormf_{n/2 + 1}
  \le \clnznormf_j' + \clnznormf_{n/2+1}'$ in part~\ref{item:prime-superset-b}
  and the final inequality is Lemma~\ref{lemma:cov-prune-ever}.
  Applying
  Lemma~\ref{lemma:cov-prune-iteration} with the
  inequality
  $\log \(\norm{\clx_j}_{\Sigma_t}^2\) + \clnznormf_j + \clnznormf_{n/2+1}
  \le \log B$ then guarantees
  that $j \not \in \removedinds_{t + 1,-i}$ as desired,
  completing the proof of claim~\ref{item:prime-superset-b}.
    
  For the proof of claim~\ref{item:prime-subset-c},
  we induct on $\removedinds_{t,-i}'$ for $t \leq T'$ and must
  account for the possibility that $\Sigma_T \npreceq
  {\Sigma}_{t}'$ even if $\removedinds_{t,-i}' \subset
  \removedinds_{T,-i}$, because $\Sigma_T$ may include the term $\clx_i
  \clx_i^T$ (i.e., $i \notin \removedinds_T$).
  The base case for $t = 0$ is trivial, so assume
  that $\removedinds'_{t,-i} \subset \removedinds_{T, -i}$.
  If $i \not\in \removedinds_T$, then Lemma~\ref{lemma:cov-prune-ever}
  and the standing assumption that $\clnznormf_i + \clnznormf_{n/2 + 1}
  \ge -\half$ guarantee that
  \begin{align*}
    \log \(\norm{\clx_j}_{\Sigma_T}^2\) \leq
    \log B - \clnznormf_i - \clnznormf_{n/2 + 1}
    \le \log B + \half
    = \log(B \sqrt{e}),
  \end{align*}
  i.e., $\norm{\clx_j}_{\Sigma_T}^2 \leq B \sqrt{e}$.
  We require the following technical observation about positive definite
  matrices, whose proof we temporarily defer.
  \begin{observation}
    \label{observation:matrix-stability}
    Let $\Sigma \in \R^{d \times d}$ be positive semi-definite, $\alpha \ge 0$,
    and $u \in \R^d$. 
    Define $\Sigma' := \Sigma - \alpha u u^T$.
    If $\norm{u}_{\Sigma}^2 \leq \frac{1}{2\alpha}$, then
    $\Sigma' \succeq \half \Sigma$ and for any $v \in \colspace(\Sigma)$,
    \begin{align*}
      \abs{\log\(\norm{v}_{\Sigma}^2\) - \log\(\norm{v}_{\Sigma'}^2 \)} \leq
      2 \alpha \norm{u}_{\Sigma}^2.
    \end{align*}
  \end{observation}      
  \noindent
  As $\norm{\clx_i}_{\Sigma_T}^2 \le B \sqrt{e}$,
  Observation~\ref{observation:matrix-stability} applies with
  $u = \clx_i$ and $\alpha = \frac{1}{n}$ when
  $n \ge 2B \sqrt{e}$,
  and thus
  \begin{equation}
    \label{eqn:index-i-not-in-removed}
    \log \(\norm{v}_{\Sigma_T - \frac{1}{n}\indic{i \notin \removedinds_T} \clx_i \clx_i^T}^2\)
    \le \log \(\norm{v}_{\Sigma_T}^2\) + \frac{2B\sqrt{e}}{n}
  \end{equation}
  for all $v$, and in particular, for  $v = \clx_j$ for each $j \in [n/2]$.
  On the other hand, regardless
  of whether $i \in \removedinds_T$, the inductive assumption that
  $\removedinds_{t,-i}' \subset \removedinds_{T,-i}$ guarantees that
  \begin{equation}
    \Sigma_t' \succeq
    \Sigma_T - \frac{1}{n}\indic{i \notin \removedinds_T} \clx_i \clx_i^T.
    \label{eqn:sigma-prime-large-enough}
  \end{equation}
  Considering $j \in \removedinds'_{t + 1, -i}$, then,
  Lemma~\ref{lemma:cov-prune-iteration}
  implies
  \begin{align*}
    \log B & < \log(\norm{\clx_j}_{\Sigma'_t}^2) + \clnznormf_j'
    + \clnznormf_{n/2 + 1}' \\
    & \stackrel{(i)}{\le}
    \log\left( \norm{\clx_j}_{\Sigma_T - \frac{1}{n}
    \indic{i \not \in \removedinds_T} \clx_i \clx_i^T}^2 \right)
    + \clnznormf_j' + \clnznormf_{n/2 + 1}' \\
    & \stackrel{(ii)}{\le}
    \log \left( \norm{\clx_j}_{\Sigma_T}^2\right)
    + \frac{2B \sqrt{e}}{n}
    + \clnznormf_j' + \clnznormf_{n/2 + 1}' \\
    & \stackrel{(iii)}{\le}
    \log \left( \norm{\clx_j}_{\Sigma_T}^2\right)
    + \clnznormf_j + \clnznormf_{n/2 + 1}.
  \end{align*}
  Here inequality~$(i)$ follows from the ordering
  relation~\eqref{eqn:sigma-prime-large-enough};
  inequality~$(ii)$ holds because if $i \in \removedinds_T$, then
  $\Sigma_T = \Sigma_T - \frac{1}{n} \indic{i \not \in \removedinds_T}
  \clx_i \clx_i^T$ and if $i \not \in \removedinds_T$ then
  inequality~\eqref{eqn:index-i-not-in-removed} holds;
  the final inequality~$(iii)$ follows by assumption
  under claim~\ref{item:prime-subset-c}.
  This gives the induction that $\removedinds'_{t + 1,-i}
  \subset \removedinds_{T, -i}$,
  as Lemma~\ref{lemma:cov-prune-iteration} shows that
  $j \in \removedinds_{T,-i}$.

  Claim~\ref{item:prime-subset-d} follows from an essentially
  identical induction argument, \emph{mutatis mutandis}, as
  that for claim~\ref{item:prime-subset-c}.

  \paragraph{Proof of Observation~\ref{observation:matrix-stability}.}
  We finally return to prove the claimed observation.  That $\Sigma'
  \succeq \half \Sigma$ follows by observing that $u \in \colspace(\Sigma)$
  and hence
  \begin{equation*}
    \Sigma - \alpha uu^T
    = \Sigma^{1/2}
    \underbrace{(I - \alpha \Sigma^{\dag/2} u u^T \Sigma^{\dag/2})}_{\succeq
      (1/2) I}
    \Sigma^{1/2}
    \succeq \half \Sigma.
  \end{equation*}
  This also implies that $\colspace(\Sigma') = \colspace(\Sigma)$.

  To prove the remainder of the lemma, it suffices to show for
  $v \in \colspace(\Sigma)$ that
  $\log(\norm{v}_{\Sigma'}^2)
  \leq \log(\norm{v}_{\Sigma}^2) +2\alpha\norm{u}_{\Sigma}^2$,
  since the other direction is immediate from $\Sigma \succeq \Sigma'$.
  Observe that
  \begin{equation*}
    (\Sigma - \alpha uu^T)^\dag
    = \Sigma^{\dag/2} (I - \alpha \Sigma^{\dag/2} u u^T \Sigma^{\dag/2})^{-1}
      \Sigma^{\dag/2}.
  \end{equation*}
  By the inequality
  $I - \alpha \Sigma^{\dag/2} u u^T \Sigma^{\dag/2}
  \succeq (1 - \alpha \norm{u}_{\Sigma}^2) I$
  we have
  \begin{align*}
    (I - \alpha \Sigma^{\dag/2} u u^T \Sigma^{\dag/2})^{-1}
    &\preceq (1 - \alpha \norm{u}_{\Sigma}^2)^{-1} I
    \preceq (1 + 2\alpha \norm{u}_{\Sigma}^2) I,
  \end{align*}
  where the final inequality follows from the assumption that
  $\norm{u}_{\Sigma}^2 \leq \frac{1}{2\alpha}$.
  Combining this with the preceding display implies that
  \begin{equation*}
    \Sigma'^\dag \preceq (1 + 2\alpha \norm{u}_{\Sigma}^2) \Sigma^\dag
  \end{equation*}
  and so
  \begin{align*}
    \log \(\norm{v}_{\Sigma'}^2\)
    & \leq \log \((1 + 2\alpha \norm{u}_{\Sigma}^2) \norm{v}_{\Sigma}^2\)
    \leq \log \(\norm{v}_\Sigma^2\) + 2 \alpha \norm{u}_\Sigma^2
  \end{align*}
  as desired.
\end{proof}

\newcommand{\eventprune}{\mc{E}_{\textup{prune}}}  
\newcommand{\eventthresh}{\mc{E}_{\textup{thr}}}  

We move to the second step we outline at the beginning of this section,
which relates the distributions of removed indices $\removedinds_T$ in the
execution of $\clipr$ on adjacent inputs $x$ and $x'$. The key idea is to
construct a deterministic map $\pi$ so that the execution of $\clipr$ on
input $x$ with noise $z$ and that on $x'$ with noise $\pi(z)$ is
similar---leveraging Lemma~\ref{lemma:deterministic-removal-relations}---and
to show that the distributions of $\pi(\clnznormr)$ and $\clnznormr$ are
similar.  Lemma~\ref{lemma:sigma-i-via-cardinality} shows that the set of
outlier indices $\removedinds_{T, -i}$ completely determines
$\est{\Sigma}_{-i}$ except in the case that $\est{\Sigma}_{-i} = \bot$,
which occurs with high probability if $\cardp{\removedinds_{T, -i}}$ is
large, so the next lemma controls the distribution of the sets of removed
indices.  To state the lemma, we require a few events whose probabilities we
can control. Recalling that $Z_j \simiid \laplace(\sigma_Z)$, define
\begin{align}
  \eventprune \defeq
  \{\clnznormf \in \R^{n/2+1} \mid
  \clnznormf_j + \clnznormf_{n/2+1} \geq -1/2 \ \txt{for all} \ j \in [n/2]\}.
  \label{eqn:event-prune}
\end{align}

To set notation for the remainder of the
proof, we shorthand the definition~\eqref{eqn:instantiated-sigma} as
\begin{equation}
  \label{eqn:sigma-and-sigma-prime}
  \begin{split}
    \est{\Sigma},
    ([\removedinds_t]_{t = 0}^T, [\Sigma_t]_{t = 0}^T, T)
    & \defeq \clipr_{\clparams}(x; \clnznormr, \clnzthrr) \\
    \est{\Sigma}',
    ([\removedinds_t']_{t = 0}^{T'}, [\Sigma_t']_{t = 0}^{T'}, T')
    & \defeq \clipr_{\clparams}(x'; \clnznormr, \clnzthrr),
  \end{split}
\end{equation}
and the definition~\eqref{eqn:instantiated-sigma-minus-i} as
\begin{equation*}
  \est{\Sigma}_{-i} = \est{\Sigma} - \frac{1}{n}
  \indic{i \not \in \removedinds_T} \clx_i \clx_i^T
  ~~ \mbox{and} ~~
  \est{\Sigma}'_{-i} = \est{\Sigma}' - \frac{1}{n}
  \indic{i \not \in \removedinds_T'} \clx_i' \clx_i'^T,
\end{equation*}
where $\bot + v = \bot$ for any vector $v$.

We have the following distributional guarantee on
the removed indices
regardless of $\clnzthrr$.
\begin{lemma}
  \label{lemma:clipr-set-probabilities}
  Let $S \subset [n/2] \setminus \{i\}$ and define 
  $\alpha = \frac{2 \sqrt{e} B (\cardp{S}+1)}{n \sigma_Z}$.
  If $\clx'_i = 0$, then
  \begin{enumerate}[label=(\alph*),labelindent=0pt]
  \item \label{item:clipr-remove-to-prime-probs}
    $
    \P(\removedinds_{T,-i} = \set{S}, \clnznormr \in \eventprune)
    \le
    \exp\left(\alpha\right)
    \P(\removedinds_{T',-i}' = \set{S})$.
  \item \label{item:clipr-prime-to-remove-probs}
    $
    \P(\removedinds_{T',-i}' = \set{S}, \clnznormr \in \eventprune)
    \le
    \exp\left(\alpha\right)
    \P(\removedinds_{T,-i} = \set{S})$.
  \end{enumerate}
\end{lemma}
\begin{proof}
  The input noise $\clnznormr$ completely determines $\removedinds_T$ and
  $\removedinds_{T'}'$ in $\clipr$ (see the \textbf{while} loop constructing
  $\removedinds_t$ in
  lines~\ref{line:clipr-while}--\ref{line:clipr-set-sigma}). Consequently,
  we may define sets of input noise $\clnznormr$ yielding a given
  set of removed indices, letting
  \begin{align*}
    \mc{Z}(\set{S})
    & \defeq \{\clnznormf \in \R^{n/2+1} \mid \removedinds_{T,-i} = \set{S}\ \ \txt{for} \ \clnznormr = \clnznormf\} \\
    \mc{Z}'(\set{S}) & \defeq \{\clnznormf \in \R^{n/2+1} \mid \removedinds_{T',-i}' = \set{S} \ \txt{for} \ \clnznormr = \clnznormf\},
  \end{align*}
  so $Z \in \mc{Z}(S)$ is equivalent to $\removedinds_{T,-i} = S$.
  It suffices to show that
  \begin{align}
    \begin{split}
      \P\(\clnznormr \in \mc{Z}(\set{S}) \cap \eventprune\)
      & \le e^\alpha  
      \P\(\clnznormr \in \mc{Z}'(\set{S})\)
      ~~ \mbox{and} \\
      \P\(\clnznormr \in \mc{Z}'(\set{S}) \cap \eventprune\)
      & \le e^\alpha
      \P\(\clnznormr \in \mc{Z}(\set{S})\),
    \end{split}
    \label{eqn:nperp-main}
  \end{align}
  as claim~\ref{item:clipr-remove-to-prime-probs}
  follows via the first bound and
  claim~\ref{item:clipr-prime-to-remove-probs} the second.

  Proceeding with the first bound, define
  $\eta \in \R^{n/2+1}$ and $\pi : \R^{n/2+1} \to \R^{n/2 + 1}$ by
  \begin{align*}
    \eta_j :=
    \begin{cases}
      2B\sqrt{e}/n  & j \in \set{S}    \\
      -2B\sqrt{e}/n & j = n/2+1          \\
      0             & \txt{otherwise},
    \end{cases}
    ~~~
    \pi(z) \defeq z + \eta.
  \end{align*}
  The deterministic removals
  Lemma~\ref{lemma:deterministic-removal-relations}
  shows that on $\clnznormr \in \eventprune$,
  if we let $\clnznormf = \clnznormr$ and $\clnznormf' = \pi(\clnznormf)$
  so that $\clnznormf_j + \clnznormf_{n/2 + 1} \ge
  \clnznormf_j' + \clnznormf_{n/2 + 1}'$ for $j \in S$,
  then parts~\ref{item:prime-superset-a} and~\ref{item:prime-subset-d}
  of Lemma~\ref{lemma:deterministic-removal-relations} give
  \begin{align*}
    \pi(\mc{Z}(\set{S}) \cap \eventprune) \subset \mc{Z}'(\set{S}).
  \end{align*}
  The first bound in~\eqref{eqn:nperp-main} then follows by the standard
  Laplacian ratio bounds in Lemma~\ref{prop:useful-4}. Indeed, we have
  $\clnznormr_j \simiid \laplace(\sigma_Z) = \laplace(\lone{\eta}
  \frac{\sigma_Z}{\lone{\eta}})$ and $\lone{\eta} = \frac{2 \sqrt{e} B (\cardp{S}+1)}{n}$.
  Then setting
  $\beta = \lone{\eta}$ yields $\beta / \sigma_Z \le \alpha$, 
  so we can apply Lemma~\ref{prop:useful-4} to obtain the claimed 
  bound~\eqref{eqn:nperp-main} via
  \begin{equation*}
    \P(Z \in \mc{Z}(S) \cap \eventprune)
    \le e^\alpha \P(Z \in \pi(\mc{Z}(S) \cap \eventprune))
    \le e^\alpha \P(Z \in \mc{Z}'(S)).
  \end{equation*}
  

  The proof of the second bound~\eqref{eqn:nperp-main}
  is essentially the same, only this time we let
  \begin{align*}
    \eta_j \defeq
    \begin{cases}
      2B\sqrt{e}/n & j \in \set{S}   \\
      0            & \txt{otherwise},
    \end{cases}
    ~~~
    \pi(z) = z + \eta.
  \end{align*}
  Then the event $\clnznormr \in \eventprune$
  implies $\pi(\clnznormr) \in \eventprune$, as
  $\pi(Z)_j + \pi(Z)_{n/2+1} \geq Z_j + Z_{n/2+1}$ for
  all $j \in [n/2]$. We may thus appeal to
  cases
  \ref{item:prime-superset-b} and
  \ref{item:prime-subset-c} of
  Lemma~\ref{lemma:deterministic-removal-relations} with
  the settings
  $\clnznormf = \pi(\clnznormr)$, $\clnznormf' = \clnznormr$ and
  $\removedinds_{T',-n}' = S$ and proceed with the same
  argument as above.
\end{proof}

The third step we outline at the beginning of the proof of
Lemma~\ref{lemma:clipr-loo-stability} is to relate the probabilities that
$\clipr$ aborts on neighboring inputs $x$ and $x'$. Recall $\est{\Sigma}$
and $\est{\Sigma}'$ are the covariances $\clipr$ outputs on inputs $x$ and
$x'$, respectively, as in definition~\eqref{eqn:sigma-and-sigma-prime}.
\begin{lemma}
  \label{lemma:clipr-perp}
  Let $\sigma_Z, \sigma_\clnzthrr > 0$,
  $\clnznormr_j \simiid \laplace(\sigma_Z)$, and
  $\clnzthrr \sim \laplace(\sigma_\clnzthrr)$
  in definition~\eqref{eqn:sigma-and-sigma-prime}.
  If $\clx_{-i} = \clx_{-i}'$ and $\clx_i' = 0$, then
  \begin{enumerate}[label=(\alph*),labelindent=0pt]
  \item \label{item:clipr-perp-a}
    $\P(\est{\Sigma} = \bot) \leq \exp(\frac{1}{\sigma_\clnzthrr})
    \P(\est{\Sigma}' = \bot)$.
  \item \label{item:clipr-perp-b}
    $\P(\est{\Sigma}' = \bot, \clnznormr \in \eventprune)
    \leq \exp(\frac{2 \sqrt{e} B}{n \sigma_Z})\P(\est{\Sigma} = \bot)$.
  \end{enumerate}
\end{lemma}
\begin{proof}
  Let $f$ denote the density of $W$, so that
  $f(w) = \frac{1}{2 \Wcovscale} \exp(-|w| / \Wcovscale)$, and thus
  $|\log \frac{f(w)}{f(w - 1)}| \le \frac{1}{\Wcovscale}$ for all $w$,
  and recall the threshold $m \in \N$ in line~\ref{line:clipr-check}
  of $\clipr$.
  Proceeding with the first claim of the lemma,
  we have the following sequence of inequalities:
  \begin{align*}
    \P\(\est{\Sigma} = \bot\) 
    & \stackrel{(i)}{\leq} \int \P\(\cardp{\removedinds_{T}}  > m + w\) f(w) dw
    \leq \int \P\(\cardp{\removedinds_{T,-i}}  > m + w - 1\) f(w) dw \\
    & \stackrel{(ii)}{\leq}
    \int \P\(\cardp{\removedinds_{T',-i}'}  > m + w - 1\) \, f(w) dw \\
    & \leq \exp(1/\Wcovscale) \int \P\(\cardp{\removedinds_{T',-i}'}  > m + w - 1\)
    f(w - 1) dw \nonumber \\
    & \stackrel{(iii)}{=} \exp(1/\Wcovscale)
    \P(\est{\Sigma}' = \bot).
  \end{align*}
  Here, step~$(i)$ follows from Lemma~\ref{lemma:sigma-i-via-cardinality}
  that $\est{\Sigma}' = \bot$ if and only if $\cardp{\removedinds_T'} > m + w$.
  Step~$(ii)$ follows from the coupling argument in
  Lemma~\ref{lemma:deterministic-removal-relations}
  part~\ref{item:prime-superset-a} : because the noise $\clnznormr$ is
  identical in both executions of $\clipr(x; \clnznormr, W)$ and $\clipr(x';
  \clnznormr, W)$, we have $\removedinds_{T,-i} \subset
  \removedinds'_{T',-i}$.
  Step~$(iii)$ applies because of Lemma~\ref{lemma:sigma-i-via-cardinality}
  again, as the assumption $\wt{x}_i' = 0$ guarantees
  $\removedinds'_{T', -i} = \removedinds'_{T'}$ (recall
  the rejection threshold in Line~\ref{line:clipr-truncate}). 
  Claim~\ref{item:clipr-perp-a} follows.
  
  For the claim~\ref{item:clipr-perp-b},
  again applying Lemma~\ref{lemma:sigma-i-via-cardinality} we have 
  \begin{align*}
    \P\(\est{\Sigma}' = \bot, \clnznormr \in \eventprune\)
    & =
    \int \P\(\cardp{\removedinds_{T',-i}'} > m + w,
    \clnznormr \in \eventprune \) f(w) dw
  \end{align*}
  and
  \begin{align*}
    \int \P\(\cardp{\removedinds_{T,-i}}  > m + w\) f(w) dw
    & \leq \P\(\est{\Sigma} = \bot\).
  \end{align*}
  Combining these displays, it thus suffices to show for all $w$ that
  \begin{align}
    \P\(\cardp{\removedinds_{T',-i}'}  > m + w, \clnznormr \in \eventprune \)
    \leq \exp\left(\frac{2 \sqrt{e} B}{n \sigma_Z}\right)
    \P\(\cardp{\removedinds_{T,-i}}  > m + w \).
    \label{eqn:clipr-set-probabilities-util1}
  \end{align}
  To this end, we adopt a similar tack as in the proof of
  Lemma~\ref{lemma:clipr-set-probabilities}, defining
  \begin{align*}
    \mc{Z}(w) & \defeq \{\clnznormf \in \R^{n/2+1} \mid
    \cardp{\removedinds_{T,-i}} > m + w \ \txt{for} \ Z = z\} \\
    \mc{Z}'(w) & \defeq \{\clnznormf \in \R^{n/2+1} \mid
    \cardp{\removedinds_{T',-i}'} > m + w \ \txt{for} \ Z = z\}
  \end{align*}
  and the single coordinate perturbation
  $\pi(z) := z + \eta$ for
  $\eta \in \R^{n/2 + 1}$ the vector with all zeros except that $\eta_{n/2 + 1}
  = \frac{2 \sqrt{e} B}{n}$.
  Similar to our proof of Lemma~\ref{lemma:clipr-set-probabilities},
  the mapping $\pi$ guarantees that $z' = \pi(z)$
  satisfies $z_j' + z_{n/2 + 1}' \le z_j + z_{n/2 + 1} - \frac{2 B \sqrt{e}}{n}$
  for all $j$, which is precisely the condition for
  case~\ref{item:prime-subset-c} of
  Lemma~\ref{lemma:deterministic-removal-relations}, and so
  $\removedinds'_{T',-i} \subset \removedinds_{T,-i}$
  irrespective of $\removedinds_{T',-i}'$ and $\removedinds_{T,-i}$.
  It thus holds simultaneously for
  all $w \in \R$ that
  \begin{align*}
    \pi\(\mc{Z}'(w) \cap \eventprune\) \subset \mc{Z}(w).
  \end{align*}
  Noting that $\lone{\eta} = \frac{2\sqrt{e} B}{n}$,
  Lemma~\ref{prop:useful-4} on likelihood ratios for Laplace random
  variables then guarantees
  \begin{equation*}
    \P(\clnznormr \in \mc{Z}'(w) \cap \eventprune)
    \le \exp\left(\frac{2 \sqrt{e} B}{n \sigma_Z}\right)
    \P(\clnznormr \in \pi(\mc{Z}'(w) \cap \eventprune))
    \le \exp\left(\frac{2 \sqrt{e} B}{n \sigma_Z}\right)
    \P(\clnznormr \in \mc{Z}(w)),
  \end{equation*}
  which is equivalent to inequality~\eqref{eqn:clipr-set-probabilities-util1}.
\end{proof}

\subsubsection{Finalizing proof of Lemma~\ref{lemma:clipr-loo-stability}}
\label{sec:proof-finally-loo-stability}

By combining Lemmas~\ref{lemma:clipr-set-probabilities} and
\ref{lemma:clipr-perp}, we can prove the stability of $\clipr$.  Recall the
set $\eventprune$ in~\eqref{eqn:event-prune} and that $W \sim
\laplace(\Wcovscale)$, and additionally define
\begin{align*}
  \eventthresh \defeq \left(-\infty, \Wcovscale \log
  \frac{1}{\gamma} \right].
\end{align*}
The key part of our argument is to show that when $x$ and $x'$ are
adjacent but $\clx_i' = 0$, if the noise variables $Z, W$ satisfy $Z \in
\eventprune$ and $W \in \eventthresh$, then for $\est{\Sigma}$ and
$\est{\Sigma}'$ as in the call~\eqref{eqn:sigma-and-sigma-prime} the
leave-one-out covariances $\est{\Sigma}_{-i}$ and $\est{\Sigma}'_{-i}$ are
similar. We then bound the probabilities of the individual events and use a
group composition argument to give the lemma for arbitrary $\clx_i'$.

With this in mind, let $A \in \R^{d \times d}$ and note that
for any fixed sample $x$, $\est{\Sigma}$ and $\est{\Sigma}_{-i}$ can take
only finitely many values. For
\begin{equation*}
  \alpha = \frac{1}{\Wcovscale} +
  \frac{2 \sqrt{e} B (\mmax+1)}{n \sigma_Z},
\end{equation*}
we show for any $A \in \R^{d \times d} \cup \{\bot\}$ that
\begin{align}
  \label{eqn:clipr-priv-main-1}
  \P\(\est{\Sigma}_{-i} = A, \clnznormr \in \eventprune, \clnzthrr \in \eventthresh\) 
  & \leq \exp(\alpha)\P\(\est{\Sigma}_{-i}' = A\)
  ~~ \mbox{and} \\
  \P\(\est{\Sigma}_{-i}' = A, \clnznormr \in \eventprune, \clnzthrr \in \eventthresh\)
  & \leq \exp(\alpha)\P\(\est{\Sigma}_{-i} = A\).
  \label{eqn:clipr-priv-main-2}
\end{align}
Lemma~\ref{lemma:clipr-perp} already implies both
inequalities~\eqref{eqn:clipr-priv-main-1} and~\eqref{eqn:clipr-priv-main-2}
hold for $A = \bot$, so all that remains is to show the same for $A \in
\R^{d \times d}$.

Proceeding first with inequality~\eqref{eqn:clipr-priv-main-1},
let $f(w) = \frac{1}{2 \Wcovscale}\exp(-|w| / \Wcovscale)$
be the density of $\clnzthrr$ and $S(A) :=
\{S \subset [n/2] \setminus \{i\} \mid A = \frac{1}{n}\sum_{j \notin S \cup
  \{i\}} \clx_j \clx_j^T\}$.  Marginalizing over $\clnzthrr$ gives
\begin{align*}
  \lefteqn{\P\(\est{\Sigma}_{-i} = A,
    \clnznormr \in \eventprune,\clnzthrr \in \eventthresh\)} \\
  & \stackrel{(i)}{=}
  \int_{\eventthresh}
  \P\(\removedinds_{T,-i} \in S(A), \cardp{\removedinds_T} \leq m + w,
  \clnznormr \in \eventprune\)
  f(w) dw \\
  & \le
  \int \P\(\removedinds_{T,-i} \in S(A),
  \cardp{\removedinds_{T,-i}} \leq \min\{m + w, \mmax\},
  \clnznormr \in \eventprune\) f(w) dw,
\end{align*}
where step~$(i)$ follows from the condition that $\cardp{\removedinds_T}
\le m + w$ if and only if $\est{\Sigma}_{-i} \neq \bot$ from
Lemma~\ref{lemma:sigma-i-via-cardinality}, and the final inequality
follows because $\eventthresh = \{w \mid w \le \mmax\}$.  Continuing, note
for each $S \in S(A)$, we can have $S = \removedinds_{T,-i}$ with
$\cardp{\removedinds_{T,-i}} \le \min\{m + w, \mmax\}$ only if $\cardp{S} \le
\min\{m + w, \mmax\} \le \mmax$, so that by
case~\ref{item:clipr-remove-to-prime-probs} of
Lemma~\ref{lemma:clipr-set-probabilities}
\begin{align*}
  \lefteqn{\P(\removedinds_{T, -i} \in S(A),
    \cardp{\removedinds_{T, -i}} \le \min\{m + w, \mmax\},
    Z \in \eventprune)} \\
  & \le
  \exp\left(\frac{2 \sqrt{e} B (\mmax+1)}{n \sigma_Z} \right)
  \P\left(\removedinds'_{T',-i} \in S(A),
  \cardp{\removedinds'_{T', -i}} \le m + w\right).
\end{align*}
Returning to the integral above, we obtain
inequality~\eqref{eqn:clipr-priv-main-1}
by integrating and applying Lemma~\ref{lemma:sigma-i-via-cardinality}:
\begin{equation*}
  \int \P\left(\removedinds'_{T',-i} \in S(A),
  \cardp{\removedinds'_{T', -i}} \le m + w\right) f(w) dw
  = \P(\est{\Sigma}'_{-i} = A)
\end{equation*}
as $\removedinds'_{T',-i} = \removedinds'_{T'}$ because
$\clx_i' = 0$ by assumption.

The proof of inequality~\eqref{eqn:clipr-priv-main-2} is essentially the
same, only now we must take additional care to account for the possibility
that $i \in \removedinds_T$.  As in the preceding integral inequalities,
Lemma~\ref{lemma:sigma-i-via-cardinality} gives
\begin{align*}
  \lefteqn{\P\(\est{\Sigma}_{-i}' = A, \clnznormr \in \eventprune,
    \clnzthrr \in \eventthresh\)} \\
  & \le \int_{\eventthresh}
  \P\(\removedinds_{T',-i}' \in S(A), \cardp{\removedinds_{T',-i}'} \leq \min\{m + w, \mmax\},
  \clnznormr \in \eventprune\) f(w) dw.
\end{align*}
In this case, with reasoning identical to that above, we apply
case~\ref{item:clipr-prime-to-remove-probs} of
Lemma~\ref{lemma:clipr-set-probabilities} to achieve
\begin{align*}
  \lefteqn{\P\(\removedinds_{T',-i}' \in S(A), \cardp{\removedinds_{T',-i}'}
    \leq \min\{m + w, \mmax\},
    \clnznormr \in \eventprune\)} \\
  & \le \exp\left(\frac{2 \sqrt{e} B (\mmax+1)}{n \sigma_Z}\right)
  \P\left(\removedinds_{T, -i} \in S(A),
  \cardp{\removedinds_{T,-i}} \le m + w\right),
\end{align*}
so
\begin{align*}
  \lefteqn{\P\(\est{\Sigma}_{-i}' = A, \clnznormr \in \eventprune,
  \clnzthrr \in \eventthresh\)} \\
  & \le \exp\left(\frac{2 \sqrt{e} B (\mmax+1)}{n \sigma_Z} \right)
  \int \P\(\removedinds_{T,-i} \in S(A),
  \cardp{\removedinds_{T,-i}} \le m + w\) f(w) dw.
\end{align*}
We upper bound the final integral by noting that
\begin{align*}
  \P\(\est{\Sigma}_{-i} = A\)
  & \stackrel{(\star)}{=} \int \P\(\removedinds_{T,-i} \in S(A),
  \cardp{\removedinds_{T}} \leq m + w\) f(w)dw \\
  & \geq \int \P\(\removedinds_{T,-i} \in S(A),
  \cardp{\removedinds_{T,-i}} \leq m + w - 1\)
  f(w) dw,
\end{align*}
where $(\star)$ follows from Lemma~\ref{lemma:sigma-i-via-cardinality},
and then using
$|\log \frac{f(w)}{f(w + 1)}| \le \frac{1}{\Wcovscale}$ for all $w$
to see that
\begin{equation*}
  \int \P\(\removedinds_{T,-i} \in S(A),
  \cardp{\removedinds_{T,-i}} \le m + w\) f(w) dw
  \le \exp\left(\frac{1}{\Wcovscale}\right)
  \P\(\est{\Sigma}_{-i} = A\),
\end{equation*}
which gives inequality~\eqref{eqn:clipr-priv-main-2}
once we substitute
$\alpha = \frac{1}{\Wcovscale} + \frac{2 \sqrt{e} B (\mmax+1)}{
  n \sigma_Z}$.

We combine inequalities~\eqref{eqn:clipr-priv-main-1}
and~\eqref{eqn:clipr-priv-main-2} to get
Lemma~\ref{lemma:clipr-loo-stability}. For any set $C
\subset \R^{d \times d} \cup \{\bot\}$,
\begin{align*}
  \P(\est{\Sigma}_{-i} \in C)
  & \le \P(\est{\Sigma}_{-i} \in C, Z \in \eventprune,
  W \in \eventthresh)
  + \P(Z \not\in \eventprune) + \P(W \not\in \eventthresh) \\
  & \le e^\alpha \P(\est{\Sigma}'_{-i} \in C)
  + \P(Z \not\in \eventprune) + \P(W \not\in \eventthresh)
\end{align*}
by inequality~\eqref{eqn:clipr-priv-main-1}. We then have the immediate
bounds
$\P(W \not\in\eventthresh) = \P(W > \Wcovscale \log \frac{1}{\gamma})
= \half \exp(-\log\frac{1}{\gamma})
= \frac{\gamma}{2}$. Similarly,
$\P(Z \not\in \eventprune)
\le \frac{n}{2} \exp(-\frac{1}{4 \sigma_Z})$
by Observation~\ref{observation:laplace-sum-tail}.
The upper bound on $\P(\est{\Sigma}'_{-i} \in C)$ is similar but
uses inequality~\eqref{eqn:clipr-priv-main-2}.

To this point, we have shown that if
$x$ and $x''$ are adjacent samples differing only in
that the difference $\clx_i = x_i - x_{n/2 + i}$ may
be non-zero while
$\clx_i'' = x''_i - x''_{n/2 + i} = 0$, then
returning to the notation~\eqref{eqn:all-instantiated-sigma}
and identifying
$\est{\Sigma}_{-i} = \est{\Sigma}_{-i}(x, Z, W)$
and $\est{\Sigma}'_{-i} = \est{\Sigma}_{-i}(x'', Z, W)$,
\begin{equation*}
  \est{\Sigma}_{-i}(x, Z, W) \close{\alpha}{\beta} \est{\Sigma}_{-i}(x'', Z, W)
\end{equation*}
for $\alpha = \frac{1}{\Wcovscale} + \frac{2 \sqrt{e} B (\mmax+1)}{n \sigma_Z}$
and $\beta = \frac{\gamma}{2} + \frac{n}{2} \exp(-\frac{1}{4 \sigma_Z})$.
Thus we obtain that if $x'$ is any sample satisfying
$x'_{-i} = x_{-i}$,
\begin{equation*}
  \est{\Sigma}_{-i}(x, Z, W) \close{\alpha}{\beta} \est{\Sigma}_{-i}(x'', Z, W)
  \close{\alpha}{\beta} \est{\Sigma}_{-i}(x', Z, W).
\end{equation*}
Using group composition (Lemma~\ref{lemma:group-privacy}),
we obtain
\begin{equation*}
  \est{\Sigma}_{-i}(x, Z, W) \close{2 \alpha}{\beta
    + e^\alpha \beta} \est{\Sigma}_{-i}(x', Z, W),
\end{equation*}
which is the desired Lemma~\ref{lemma:clipr-loo-stability}.

\section{Proofs for mean estimation}\label{sec:mean}
In this section, we provide the proofs of Lemmas~\ref{lemma:mean-sample}
and~\ref{lemma:mean-covariance}.
Throughout, we differentiate outputs of
$\meansafe$ on inputs $x$ versus $x'$ (or $\genpsd$ versus $\genpsd'$) via tick
marks, so that (for example) $\msmean$ corresponds to the mean in
Line~\ref{line:ms-mean-def} of $\meansafe$ on input sample $x$, or
$D_j'$ corresponds to the log-diameter in
Line~\ref{line:ms_m-ass} of $\meansafe$ on input sample $x'$.
We will make this precise using the function $\Transcript(x, \genpsd)$
from \eqref{eqn:meansafe-output}, which is the transcript $\meansafe$
outputs on input $x, \genpsd$.  

\subsection{Proof of Lemma~\ref{lemma:mean-sample}}\label{sec:mean-sample-pf}
We shorthand $\msout(x, \genpsd)$ and $\msout(x', \genpsd)$ as $\msout$ and
$\msout'$ respectively, and unpack the corresponding
execution transcripts:
\begin{align*}
  (D, \wt{D}, R, t, \msmean) \defeq \Transcript(x, \genpsd) 
  ~~~ \mbox{and} ~~~
  (D', \wt{D}', R', t', \msmean') \defeq \Transcript(x', \genpsd).
\end{align*}
Throughout our arguments, $i \in [n]$ denotes the index at which the samples
$x, x'$ differ, that is, $x_{-i} = x_{-i}'$ while we may have $x_i \neq
x_i'$.

The main idea in the proof of Lemma~\ref{lemma:mean-sample} is to first bound
the sensitivity of the mean, showing that (with high probability)
$\norm{\msmean - \msmean'}_A$ is small, unless there are too many outlying
entries $x_j$.
We do this in Lemma~\ref{lemma:ms-sens} by showing that for appropriate subgroup
sizes $b$ (recall the random partition $\msnzgf$ of $[n]$ into blocks of
size $n/b$ in $\meansafe$), the $\meansafe$ algorithm correctly identifies
all outliers without pruning many inlying datapoints.
In the second step, we finalize the proof (section~\ref{sec:ms-final}) by
combining the sensitivity bound with more or less standard distributional
stability guarantees for Gaussian distributions, which we list in
the preliminary section~\ref{sec:prelims}.


We begin by formalizing two properties that will be helpful to
proving the sensitivity bound in Lemma~\ref{lemma:ms-sens}.
We recall the notation $t$ (respectively $t'$) for denoting the
number of pruned groups in
Lines~\ref{line:ms_rejection-for}--\ref{line:ms_end-rejection-for} of
$\meansafe$ on inputs $x$ and $x'$, while $\removedinds$ and $\removedinds'$
denote the sets of all pruned indices. 
Of the next two lemmas, Lemma~\ref{lemma:ms-R-stability} bounds differences
between $\removedinds$ and $\removedinds'$ and $t$ and $t'$, while
Lemma~\ref{lemma:sample-mean-bd} is a generic lemma that bounds the
difference of empirical means with nested index sets.
These two lemmas are combined in Lemma~\ref{lemma:ms-sens} to bound
the difference between the estimated mean $\msmean = \frac{1}{n -
  \cardp{\removedinds}} \sum_{j \not \in \removedinds} x_j$ and $\msmean'$.

Before stating Lemma~\ref{lemma:ms-R-stability}, recall for sets $S,S'$ that
$\dsym(S, S') = \max\{\cardp{S \setminus S'}, \cardp{S' \setminus S}\}$.

\begin{lemma}[Stability of rejected indices]
  \label{lemma:ms-R-stability}
  Let $t, t'$ and $\removedinds, \removedinds'$ be as above. Then $|t-t'|
  \leq 1$ and $\dsym(\removedinds,\removedinds')\leq b$.
\end{lemma}
\begin{proof}
  Let the set $J \defeq \{j \mid \wt{D}_j \neq \bot, \wt{D}_j \geq
  \log(\sqrt{B}/4)\}$ index the subgroups pruned by the execution of $\meansafe$ 
  on the sample $x'$, and similarly define $J'$ relative to $\wt{D}'$ for the sample $x'$. 
  Then $t = \cardp{J}$ and $\removedinds = \cup_{j \in J}
  \partitionset_j$, and also $t' = \cardp{J'}$ and 
  $\removedinds' = \cup_{j \in J'}\partitionset_j'$. We show $\dsym(J,J')
  \leq 1$, from which the claim $|t-t'| \leq 1$ follows immediately and the
  claim $\dsym(\removedinds,\removedinds')\leq b$ follows from the fact that
  $\cardp{\partitionset_j} = b$ for all $j \in [n/b]$.

  Recalling $x$ and $x'$ differ only at index $i$, suppose that
  $i \in \partitionset_\ell$ for $\ell \in [n/b]$.  Then
  $x_{\partitionset_j} = x_{\partitionset_j}'$ for all $j \neq \ell$; in
  particular, $\diam_A(x_{\partitionset_j}) = \diam_A(x_{\partitionset_j}')$
  and so $D_j = D_j'$ for $j \neq \ell$.  Thus, the indices of the $k$
  largest elements of $D + \msnztopir$ and $D' + \msnztopir$, i.e.,
  those subgroups identified by $\topkdp$ as having the largest diameters,
  which we
  denote by $K = \{j \mid \wt{D}_j \neq \bot\}$ and $K' = \{j \mid \wt{D}_j'
  \neq \bot\}$ respectively, differ by at most one index: $\dsym(K,K') \leq
  1$ with equality obtaining only if $\ell$ is in exactly one of $K$ or
  $K'$.  If $\ell$ is in neither $K$ nor $K'$, then $J = J'$ and the claim
  $\dsym(J,J') \leq 1$ follows.  Otherwise, supposing $\ell \in K$, the
  bound $\dsym(K,K') \leq 1$ implies $K \setminus \{\ell\} \subset K'$ and
  thus $\wt{D}_{K \setminus \{\ell\}} = \wt{D}_{K \setminus \{\ell\}}'$, or
  vice versa if $\ell \in K'$; $\dsym(J,J') \leq 1$ then follows from $J
  \subset K$ and $J' \subset K'$.
\end{proof}
\begin{lemma}\label{lemma:sample-mean-bd}
  Let $\{y_1, \ldots, y_n\}$ be an arbitrary collection of vectors and
  $S \subset S' \subset [n]$.
  Define
  $\mu_S \defeq \frac{1}{\cardp{S}} \sum_{i \in S} y_i$ and
  $\mu_{S'} \defeq \frac{1}{\cardp{S'}} \sum_{i \in S'} y_i$.
  Then 
  \begin{align*}
    \norm{\mu_S - \mu_{S'}}
    \leq \frac{\cardp{S' \setminus S} \diam_{\norm{\cdot}}(y_{S'})}{\cardp{S'}}.
  \end{align*}
\end{lemma}
\begin{proof}
  Observe
  \begin{align*}
    \mu_S - \mu_{S'}
    = \mu_S - \(\frac{\cardp{S}}{\cardp{S'}} \mu_S
      + \frac{1}{\cardp{S'}} \sum_{i \in S' \setminus S} y_i\)
    = \frac{1}{\cardp{S'}} \sum_{i \in S' \setminus S} (\mu_S - y_i),
  \end{align*}
  where from the assumption that $S \subset S'$ we have
  \begin{align*}
    \max_{i \in S' \setminus S} \norm{y_i - \mu_S}
    \leq \max_{j \in S, i \in S'} \norm{y_i - y_j} 
    \leq \diam_{\norm{\cdot}}(S').
  \end{align*}
  The claim then follows as
  $\norm{\mu_S - \mu_{S'}}
    \leq \frac{1}{\cardp{S'}} \sum_{i \in S' \setminus S} \diam_{\norm{\cdot}}(y_{S'})
    = \frac{\cardp{S' \setminus S} \diam_{\norm{\cdot}}(x_{S'})}{\cardp{S'}}$.
\end{proof}

We now turn to the first step we outline, providing an explicit bound
on $\norms{\msmean - \msmean'}_A$ except on the event that
$\max\{t, t'\} = k$. Recall the definition
$\Delta = \frac{5b \sqrt{B}}{2n}
\exp(3\mstopkscale \log \frac{2n}{b\gamma})$ in the statement
of Lemma~\ref{lemma:mean-sample}.

\newcommand{\jopt}{j^\star}

\begin{lemma}\label{lemma:ms-sens}
  With probability at least $1 - \gamma - n^2 2^{1-b}$,
  $\max\{t,t'\} = k$ or $\norm{\msmean - \msmean'}_A \leq \Delta$.
\end{lemma}
\begin{proof}
  We first show that with probability at least
  $1 - n^2 2^{-b}$ over the random partition
  $\partition \sim \uniform(\Partitions_{n,b})$,
  $\partition = (\partitionset_1, \ldots, \partitionset_{n/b})$,
  \begin{align}
    \diam_A(x_{\removedinds^c}) \leq
    \half \exp(2\norm{\msnztopir}_\infty + \norm{\msnztopiir}_\infty) \sqrt{B},
    \label{eqn:ms-sens-claim1}
  \end{align}
  with the same bound holding for $x'$ by symmetry. 
  To this end, observe that for the index set
  \begin{equation*}
    J
    \defeq \{j \in [n/b]
    \mid \wt{D}_j \neq \bot, \wt{D}_j \geq \log(\sqrt{B}/4)\},
  \end{equation*}
  $\meansafe$ constructs the removed indices $\removedinds$ in
  Lines~\ref{line:ms_rejection-for}--\ref{line:ms_end-rejection-for} via
  the union
  $\removedinds = \cup_{j \in J} \partitionset_j$.  The first step in the
  bound~\eqref{eqn:ms-sens-claim1} is to bound the diameter of the set
  $x_{\removedinds^c}$ by the diameters of the constituent sets within
  $\removedinds$, which the following generic lemma allows (see
  Section~\ref{sec:proof-generic-diameter-sampling} for a proof).
  \begin{claim}
    \label{claim:generic-diameter-sampling}
    Let $\{y_1, \ldots, y_n\}$ be an arbitrary collection of vectors
    and $\partition \sim \uniform(\Partitions_{n, b})$.
    With probability at least $1 - n^2 2^{- b}$, for all index sets
    $J \subset [n/b]$, the set $\partitionset_J \defeq \cup_{j \in J}
    \partitionset_j$ satisfies $\diam_{\norm{\cdot}}(y_{\partitionset_J})
    \leq 2 \max_{j \in J} \diam_{\norm{\cdot}}(y_{\partitionset_j})$.
  \end{claim}

  In light of
  Claim~\ref{claim:generic-diameter-sampling},
  inequality~\eqref{eqn:ms-sens-claim1} follows by showing
  \begin{equation}
    \diam_A(x_{\partitionset_j})
    \leq \exp(2\norm{\msnztopir}_\infty + \norm{\msnztopiir}_\infty) \sqrt{B}/4
    \label{eqn:ms-diam-bd}
  \end{equation}
  for all $j \not\in J$ on the event $t < k$.  When $t < k$, there exists an
  index $\ell \in [n/b]$ such that $\wt{D}_\ell \neq \bot$ and $\wt{D}_\ell
  \leq \log(\sqrt{B}/4)$, i.e., $\ell$ indexes one of the $k$ largest
  elements of $D + \msnztopir$ but $\ell \notin J$.  Thus, for $j \notin J$
  such that $\wt{D}_j = \bot$, i.e., $\log(\diam_A(x_{\partitionset_j})) +
  \msnztopir_j$ is not among the $k$ largest elements of $D + \msnztopir$
  (by the construction in $\topkdp$), we have
  \begin{align*}
    \log(\diam_A(x_{\partitionset_j}))
    \leq \log(\diam_A(x_{\partitionset_\ell})) + 2\norm{\msnztopir}_{\infty}.
  \end{align*}
  Meanwhile, for all $j \notin J$ such that $\wt{D}_j \neq \bot$, including
  $j = \ell$, from the definition of $J$ we immediately have
  \begin{align*}
    \log(\diam_A(x_{\partitionset_j})) + \msnztopiir_j \leq \log(\sqrt{B}/4).
  \end{align*}
  The claim~\eqref{eqn:ms-diam-bd}, and hence claim~\eqref{eqn:ms-sens-claim1},
  thus follows from the preceding two displays.
  Moreover, via a union bound over the two executions of $\meansafe$,
  Claim~\ref{claim:generic-diameter-sampling} gives
  \begin{equation}
    \max\{\diam_A(x_{R^c}),\diam_A(x_{R'^c}')\} \le
    \exp(2\norm{\msnztopir}_\infty + \norm{\msnztopiir}_\infty)
    \frac{\sqrt{B}}{2}
    ~~ \mbox{or}~~ \max\{t, t'\} = k
    \label{eqn:ms-sens-ub}
  \end{equation}
  with probability at least $1 - n^2 2^{1-b}$.
  
  We can now bound $\norm{\msmean - \msmean'}_A$
  for $\msmean = \frac{1}{n - \cardp{\removedinds}}
  \sum_{j \notin \removedinds} x_j$ and
  $\msmean = \frac{1}{n - \cardp{\removedinds'}}
  \sum_{j \notin \removedinds'} x_j'$
  via the following claim (essentially, a number of applications
  of the triangle inequality),
  whose proof we also defer (see Section~\ref{sec:proof-means-triangle-inequality}).
  \begin{claim}
    \label{claim:means-triangle-inequality}
    $\norm{\msmean - \msmean'}_A
      \leq \frac{4(b + 1)}{n} \max \{\diam_A(x_{R^c}), \diam_A(x_{R'^c}')\}$.
  \end{claim}
  
  Using Claim~\ref{claim:means-triangle-inequality}, the
  main Lemma~\ref{lemma:ms-sens}
  follows relatively quickly.
  By combining the
  display~\eqref{eqn:ms-sens-ub} with the fact that, by elementary
  calculation,
  \begin{align*}
    \P(\max\{\norm{\msnztopir}_\infty, \norm{\msnztopiir}_\infty\} >
    \mstopkscale \log(2n/b\gamma)) \leq \gamma,
  \end{align*}
  we obtain that with probability at least
  $1 - \gamma - n^2 2^{1-b}$, $\max\{t, t'\} = k$ or
  \begin{equation*}
    \norm{\msmean' - \msmean'}_A
    \le \frac{2(b + 1) \sqrt{B}}{n}
    \exp\left(2 \linf{\msnztopir} + \linf{\msnztopiir}\right)
    \le \frac{2(b + 1) \sqrt{B}}{n}
    \exp\left(3 \mstopkscale \log \frac{2n}{b \gamma}\right).
  \end{equation*}
  Recalling the assumption that the batchsize $b \ge 4$
  (so $2(b + 1) \le \frac{5}{2} b$),
  we obtain the lemma.
\end{proof}

\subsubsection{Finalizing proof of
  Lemma~\ref{lemma:mean-sample}}
\label{sec:ms-final}

We prove for any (measurable) event $O \subset \R^d \cup \{\bot\}$ that
\begin{align}
  \P(\msout \in O) \leq e^{\alpha + 1/\Wmeanscale} \P(\msout' \in O) + \beta_1 + \beta_2,
    \label{eqn:ms-proof-claim}
\end{align}
where $\alpha > 0$ and $\beta_1 \in (0, 1)$
determine the Gaussian noise scale for
$\msnznr \sim \normal(0, \msnscale^2 I)$ via
\begin{equation*}
  \msnscale =
  \begin{cases}\frac{\Delta}{\alpha}
    \sqrt{1.25 \log \frac{1}{\beta_1}}
    & \mbox{if}~ \alpha \le 1 \\
    \frac{\Delta}{\sqrt{2 \log \frac{1}{\beta_1} + 2 \alpha} - \sqrt{2 \log
        \frac{1}{\beta_1}}}
    & \mbox{otherwise},
  \end{cases}
  ~~ \mbox{and} ~~
  \beta_2 = \half e^{-(k/3 - 1) / \Wmeanscale} + \gamma + n^2 2^{1 - b}.
\end{equation*}
The other direction follows by symmetry.
We treat $O \subset \R^d$ and $O = \perp$ separately, merging the two cases 
at the end to show the claim~\eqref{eqn:ms-proof-claim}.
Supposing first $O \subset \R^d$, the following observation delineates necessary and sufficient conditions for $\pmout \in O$.
\begin{observation}\label{obs:ms-proof}
  Let $O \subset \R^d$. Then $\pmout \in O$ if and only if $t \leq 2k/3 + W$ and 
  $\msmean + A^{1/2}\msnznr \in O$. 
\end{observation}
\begin{proof}
  From the condition for returning $\bot$ in Line~\ref{line:ms_bot-output} of $\meansafe$, we immediately have 
  $\pmout = \bot \notin \R^d$ if and only if $t > 2k/3 + W$; thus, the condition $t \leq 2k/3 + W$
  is necessary and sufficient for $\pmout \in \R^d$.
  As either $\pmout = \bot$ or $\pmout = \msmean + A^{1/2}\msnznr$ by definition, 
  it then follows trivially that $t \leq 2k/3 + W$ and
  $\msmean + A^{1/2}\msnznr \in O$ together suffice to obtain $\pmout \in O$.
\end{proof}

Marginalizing over the number of sets of rejected indices $t$ and $\msmean$
we have the following sequence of inequalities:
\begin{align}
  & \P(\msout \in O) \nonumber \\
  & = \Ep \[\P(\msmean + A^{1/2}\msnznr \in O
  \mid \msmean) \P(t \leq 2k/3 + W \mid t) \] \nonumber \\
  & \stackrel{(i)}{\leq}
  \Ep \[ \P(\msmean + A^{1/2}\msnznr \in O \mid \msmean)
  \P(t \leq 2k/3 + W \mid t) 
  \indic{\norm{\msmean' - \msmean'}_A \leq \Delta} \]
  \nonumber \\
  & \qquad + \Ep \[\P(t \leq 2k/3 + W \mid t) \indic{\max\{t,t'\} = k} \]
  + \gamma + n^2 2^{1-b} \nonumber \\
  & \stackrel{(ii)}{\leq } \Ep \[\P(\msmean + A^{1/2}\msnznr \in O \mid \msmean)
  \P(t \leq 2k/3 + W \mid t) 
  \indic{\norm{\msmean' - \msmean'}_A \leq \Delta}\] \nonumber \\
  & \qquad + \P(W \geq k/3 - 1) + \gamma + n^2 2^{1-b} \nonumber \\
  & =
  \Ep \[\P(\msmean + A^{1/2}\msnznr \in O
  \mid \msmean)
  \P(t \leq 2k/3 + W \mid t) \indic{\norm{\msmean' - \msmean'}_A \leq \Delta} \]
  + \beta_2 \label{eqn:ms-proof-disp1}
\end{align}
Here, step $(i)$ follows because $\norm{\msmean' - \msmean'}_A \leq \Delta$
or $\max\{t,t'\} = k$ occurs with probability at least $1 - \gamma-n^2
2^{1-b}$ by Lemma~\ref{lemma:ms-sens}; step $(ii)$ because $\abs{t - t'} \leq
1$ by Lemma~\ref{lemma:ms-R-stability} and so $\max\{t,t'\} = k$ implies $t
\geq k-1$; the final equality follows from the
identity $\P\(W \geq k/3 - 1\) =
\frac{1}{2}e^{-(k/3-1)/\Wmeanscale}$ and definition of $\beta_2$.

Continuing, we can bound the last expectation in the preceding display by 
\begin{align}
  \lefteqn{\Ep
    \[\P(\msmean + A^{1/2}\msnznr \in O \mid \msmean)
    \P(t \leq 2k/3 + W \mid t) \indic{\norm{\msmean' - \msmean'}_A \leq \Delta}\]} \nonumber \\ 
  & \stackrel{(i)}{\leq}
  \exp(\alpha)\Ep \[\P(\msmean' + A^{1/2}\msnznr \in O
  \mid \msmean') \P(t \leq 2k/3 + W \mid t) \] + \beta_1 \nonumber\\
  & \stackrel{(ii)}{\leq} \exp(\alpha + 1/\Wmeanscale) \Ep
  \[\P(\msmean' + A^{1/2}\msnznr \in O \mid \msmean')
  \P(t' \leq 2k/3 + W \mid t') \] + \beta_1 \nonumber\\
  & = \exp(\alpha + 1/\Wmeanscale) \P(\msout' \in O) + \beta_1
    \label{eqn:ms-proof-disp2},
\end{align}
with step $(i)$ following from the privacy of the Gaussian mechanism with
noise $\sigma_N$ and sensitivity bound $\norms{\msmean - \msmean'} \le
\Delta$ (Lemma~\ref{lemma:normal-closeness-mean}); step $(ii)$ from $\abs{t
  - t'} \leq 1$ by Lemma~\ref{lemma:ms-R-stability} and that $W \sim
\Lap(\Wmeanscale)$; and the final equality follows directly from
Observation~\ref{obs:ms-proof}, applied here to the execution of $\meansafe$
on data $x'$.
Combining inequalities~\eqref{eqn:ms-proof-disp1} and
\eqref{eqn:ms-proof-disp2} yields the claim~\eqref{eqn:ms-proof-claim}
when $O \subset \R^d$.

For the case that $O = \{\perp\}$, we have 
\begin{align*}
  \P(\msout = \bot)
  & = \Ep \[\P(t > 2k/3 + W \mid t)\] \\
  & \leq
  e^\frac{1}{\Wmeanscale} \Ep \[\P(t' > 2k/3 + W \mid t')\] 
  = \exp^\frac{1}{\Wmeanscale} \P(\msout = \bot).
\end{align*}
Here, the two equalities follow from the condition for returning $\bot$ in
Line~\ref{line:ms_bot-output} of $\meansafe$, while the inequality follows
because $\abs{t - t'} \le 1$ by Lemma~\ref{lemma:ms-R-stability} and that $W
\sim \Lap(\Wmeanscale)$.  The
claim~\eqref{eqn:ms-proof-claim} for arbitrary $O$ is immediate.





\subsubsection{Proof of Claim~\ref{claim:generic-diameter-sampling}}
\label{sec:proof-generic-diameter-sampling}

Consider the event $\cE$ that for all indices $i_1, i_2 \in [n]$, with $i_1
\in \partitionset_{j_1}$ and $i_2 \in \partitionset_{j_2}$, we have 
$\norm{y_{i_1} -
  y_{i_2}} \leq 2 \max\{\diam(y_{\partitionset_{j_1}}),
\diam(y_{\partitionset_{j_2}})\}$.  The claim holds on
$\cE$: for any $J \subset [n/b]$ and $\partitionset_J \defeq \cup_{j \in J}
\partitionset_j$, there exist $j_1,j_2 \in J$ with $i_1 \in \partitionset_{j_1}$
and $i_2 \in \partitionset_{j_2}$
attaining
$\diam(y_{\partitionset_J}) = \norm{y_{i_1}
  - y_{i_2}}$, and so 
\begin{align*}
  \norm{y_{i_1} - y_{i_2}}
  \leq 2 \max\{\diam(y_{\partitionset_{j_1}}),
  \diam(y_{\partitionset_{j_2}})\}
  \leq 2 \max_{j \in J} \diam(y_{\partitionset_{j}}).
\end{align*}
It remains to show that $\cE$ occurs with probability at least $1-n^2
2^{-b}$.  As there are $\binom{n}{2} \le \half n^2$ unordered pairs of
distinct indices $i_1, i_2 \in [n]$, the result obtains from a union bound
if we show that $\norm{y_{i_1} - y_{i_2}} > 2
\max\{\diam(y_{\partitionset_{j_1}}), \diam(y_{\partitionset_{j_2}})\}$
occurs with probability at most $2^{1-b}$.

Proceeding, let $i_1, i_2 \in [n]$ and $i_1 \in \partitionset_{j_1}, i_2
\in \partitionset_{j_2}$ and let $c = \frac{1}{2}\norm{y_{i_1} -
  y_{i_2}}$.  If $i_1 = i_2$ or $j_1=j_2$, there is nothing to show, so
assume $i_1 \neq i_2$ and $j_1 \neq j_2$.  Let $C_1 = \{i \in [n]
\setminus \{i_1, i_2\} \mid \norm{y_{i_1} - y_{i}} < c\}$ and $C_2 = \{i
\in [n] \setminus \{i_1, i_2\} \mid \norm{y_{i_2} - y_{i}} < c\}$
be those indices $i$ for which $y_i$ is close to $y_{i_1}$ or $y_{i_2}$,
respectively.
By the triangle inequality,
$C_1$ is disjoint from $C_2$, and so
without loss of generality, we suppose that $\cardp{C_1} \leq (n-2)/2$.

We wish to show that $\diam(y_{\partitionset_{j_1}}) \ge c$, for which it
is sufficient that there exists an index in $\partitionset_{j_1} \setminus
\{i_1\}$ not in $C_1$. So
by showing $\partitionset_{j_1} \setminus \{i_1\} \subset C_1$ occurs with
probability at most $2^{1-b}$, we will be done.
As $\partition \sim \uniform(\Partitions_{n, b})$, the set
$\partitionset_{j_1} \setminus \{i_1\}$ is a uniformly distributed subset of
$[n] \setminus \{i_1, i_2\}$ of size $b-1$.
Consequently, there are $\binom{n-2}{b-1}$ distinct values it can take
and $\binom{\cardp{C_1}}{b-1}$ values such that
$\partitionset_{j_1} \setminus \{i_1\} \subset C_1$.
Therefore, the probability that $\partitionset_{j_1} \setminus \{i_1\} \subset C_1$ is
\begin{equation*}
  \P(\partitionset_{j_1} \setminus \{i_1\} \subset C_1)
  = \frac{\binom{\cardp{C_1}}{b-1}}{\binom{n-2}{b-1}}
  = \prod_{i=0}^{b-2} \frac{\hinge{\cardp{C_1} - i}}{n-2-i}
  \leq \left(\frac{\cardp{C_1}}{n-2}\right)^{b-1}
  \leq 2^{1-b},
\end{equation*}
where the last inequality follows because $\cardp{C_1} \leq (n-2)/2$.

\subsubsection{Proof of Claim~\ref{claim:means-triangle-inequality}}
\label{sec:proof-means-triangle-inequality}

\newcommand{\removedall}{\removedinds_{\textup{all}}}
\newcommand{\muall}{\est{\mu}_{\textup{all}}}

Recall that
\begin{equation*}
  \msmean = \frac{1}{n-\cardp{\removedinds}}
  \sum_{j \not\in \removedinds} x_j \quad\text{and}\quad
  \msmean' = \frac{1}{n-\cardp{\removedinds'}}
  \sum_{j \not\in \removedinds'} x'_j,
\end{equation*}
and define
\begin{align*}
  \removedall   &\defeq \removedinds \cup \removedinds', \quad
  \muall  \defeq \frac{1}{n - \cardp{\removedall}}
  \sum_{j \not\in \removedall} x_j, \quad
  \muall' \defeq \frac{1}{n - \cardp{\removedall}}
  \sum_{j \not\in \removedall} x_j'.
\end{align*}
Lemma~\ref{lemma:ms-R-stability} gives
$\cardp{\removedinds^c \setminus \removedall^c}
= \cardp{\removedall \setminus \removedinds} \leq b$, and by assumption
on the batchsize $b$ and rejection threshold $k$ we also have
$\cardp{\removedall} \leq b + \cardp{\removedinds}
\leq b + kb \leq \tfrac{n}{2}$.

Applying Lemma~\ref{lemma:sample-mean-bd} with $S=\removedall^c$, and
$S'=\removedinds^c$, we get 
\begin{equation*}
  \norm{\msmean - \muall}_{\genpsd}
  \leq \frac{\cardp{\removedinds^c \setminus \removedall^c}
    \diam_A(x_{\removedinds^c})}{\cardp{\removedinds^c}}
  \leq \frac{2b\diam_A(x_{\removedinds^c})}{n}
\end{equation*}
as $\cardp{\removedinds} \leq \cardp{\removedall} \leq \frac{n}{2}$.  Applying
Lemma~\ref{lemma:sample-mean-bd} again, this time with dataset $x'$,
$S=\removedall'^c$ and $S'=\removedinds'^c$, we get $\norm{\msmean' -
  \muall'}_{\genpsd} \leq \frac{2b}{n}\diam_A(x'_{\removedinds'^c})$.

Now we bound $\muall-\muall'$, where recalling that
index $i$ is the sole (potentially) differing index in $x, x'$ (that is,
$x_{-i} = x_{-i}'$), we can write as
\begin{equation*}
  \muall-\muall'
  = \frac{1}{n - \cardp{\removedall}} \sum_{j \not\in \removedall} (x_j - x'_j)
  = \frac{\1{i \not\in \removedall}}{n-\cardp{\removedall}} (x_i - x'_i).
\end{equation*}
If $i \in \removedall$, this difference is $0$.  Otherwise, $i \not\in
\removedinds$ and $i \not\in \removedinds'$.  As $\cardp{\removedall} \leq
\frac{n}{2}$, we may pick some $j' \not\in \removedall\cup\{i\}$.  Because
$x_{j'}=x'_{j'}$, we have both $\norm{x_i - x_{j'}}_{\genpsd} \leq
\diam_A(x_{\removedinds^c})$ and both $\norm{x'_i - x_{j'}}_{\genpsd} \leq
\diam_A(x'_{\removedinds'^c})$.  The triangle inequality then gives
$\norm{x_i - x'_i}_{\genpsd} \leq 2\max\{\diam_A(x_{\removedinds^c}),
\diam_A(x'_{\removedinds'^c})\},$ and so $\lv \muall-\muall' \rv_{\genpsd}
\leq \frac{4}{n}\max\{\diam_A(x_{\removedinds^c}),
\diam_A(x'_{\removedinds'^c})\}$.
Combining the above, the claim follows immediately from
\begin{align*}
  \norm{\msmean - \msmean'}_\genpsd
  & \le \norm{\msmean - \muall}_\genpsd
  + \norm{\muall - \muall'}_\genpsd
  + \norm{\muall' - \msmean'}_\genpsd \\
  & \le \frac{2 b \diam_A(x_{\removedinds^c})}{n}
  + \frac{4 \max\{\diam_A(x_{\removedinds^c}),
    \diam_A(x'_{\removedinds'^c})\}}{n}
  + \frac{2 b \diam_A(x'_{\removedinds'^c})}{n}.
\end{align*}

\subsection{Proof of Lemma~\ref{lemma:mean-covariance}}
  \label{sec:proof-mean-covariance}
Unpacking the execution transcripts
$\Transcript(x, \genpsd)$ and $\Transcript(x, \genpsd')$ from
\eqref{eqn:meansafe-output} as
\begin{equation*}
  (D, \wt{D}, R, t, \msmean) \defeq \Transcript(x, \genpsd) 
  ~~~ \mbox{and} ~~~
  (D', \wt{D}', R', t', \msmean') \defeq \Transcript(x, \genpsd'),
\end{equation*}
observe that given the pair $(\wt{D}, A^{1/2} \msnznr)$,
$\msout(x, A)$ is independent of $A$ (see the execution of $\meansafe$), and
analogously, $\msout(x, A')$ is independent of $A'$ given
$(\wt{D}', \genpsd'^{1/2}\msnznr)$.  Therefore, by showing
$\genpsd^{1/2}\msnznr \close{\alpha_1}{\beta} \genpsd'^{1/2}\msnznr$ and
$\wt{D} \close{\alpha_2}{0} \wt{D}'$, basic composition
(Lemma~\ref{lemma:joint-closeness}) and the post-processing property
(Lemma~\ref{lemma:post}) will imply the
claimed result that $\msout \close{\alpha_1 +
  \alpha_2}{\beta} \msout'$.

Recalling $\msnznr \sim \normal(0,\msnscale^2 I)$, we have
$\genpsd^{1/2}\msnznr \sim \normal(0, \msnscale^2 \genpsd)$ and
$\genpsd'^{1/2}\msnznr \sim \normal(0, \msnscale^2 \genpsd')$,
and so $\genpsd^{1/2}\msnznr \close{\alpha_1}{\beta} \genpsd'^{1/2}\msnznr$
follows immediately from the
assumption $\dpsd(\genpsd, \genpsd') \leq \frac{a}{n}$ and the closeness
of Gaussian distributions with differing covariances
(Lemma~\ref{lemma:normal-closeness-cov}).

To show $\wt{D} \close{\alpha_2}{0} \wt{D}'$,
we make the following observation to bound the sensitivity of the
log-Mahalanobis norm for $\genpsd$ and $\genpsd'$.
\begin{observation}
  \label{observation:log-norm-sens}
  Suppose $\genpsd, \genpsd' \in \R^{d\times d}$ and
  $\dpsd(\genpsd, \genpsd') \leq \gamma < \infty$.
  Then for any $v \in \colspace(\genpsd)$,
  $| \log \norm{v}_\genpsd - \log \norm{v}_{\genpsd'} | \leq \gamma/2$.
  For any $v \not \in \colspace(\genpsd)$,
  $\log \norm{v}_\genpsd = \log \norm{v}_{\genpsd'} = \infty$.
\end{observation}
\begin{proof}
  Observe $\dpsd(\genpsd, \genpsd') < \infty$ trivially implies $\genpsd$ and $\genpsd'$
  are PSD and their columnspaces coincide, from which the second claim
  immediately follows.
  For $v \in \colspace(\genpsd)$, we only show
  $\log \norm{v}_{\genpsd'} \leq \frac{1}{2}\gamma + \log \norm{v}_{\genpsd}$,
  as the reverse inequality holds by symmetry.
  By assumption,
  \begin{equation*}
    \opnorm{\genpsd'^{\dag/2} (\genpsd - \genpsd') \genpsd'^{\dag/2}}
    \leq \dpsd(\genpsd, \genpsd') \leq \gamma
  \end{equation*}
  and hence $\genpsd'^{\dag/2} (\genpsd - \genpsd') \genpsd'^{\dag/2}
    \preceq \gamma I$.
  Conjugating by $\genpsd'^{1/2}$ and rearranging terms, we have
  $\Pi_{\genpsd'} \genpsd \Pi_{\genpsd'} \preceq (1+\gamma)\genpsd'$.
  Because $\Pi_{\genpsd'} = \Pi_{\genpsd}$, we have
  $\Pi_{\genpsd'} \genpsd \Pi_{\genpsd'} = \genpsd$,
  which yields $\genpsd \preceq (1+\gamma)\genpsd'$, or equivalently
  $\genpsd'^\dag \preceq (1+\gamma) \genpsd^\dag$.
  Therefore $\norm{v}_{\genpsd'}^2 \leq (1+\gamma) \norm{v}_\genpsd^2$.
  Taking square roots and logarithms on both sides proves the claim 
  as $\log(\sqrt{1+\gamma}) \leq \frac{\gamma}{2}$.
\end{proof}
\noindent
This observation, coupled with our construction
that both $\msout(x, \genpsd)$ and $\msout(x, \genpsd')$
use the same (random) partition
$\partition = (\partitionset_1, \ldots, \partitionset_{n/b})$,
implies
$|D_j - D'_j| \leq \frac{a}{2n}$ for all $j \in [n/b]$;
hence $\linf{D - D'} \leq \frac{a}{2n}$
(the indices where the entries are infinite coincide).
The closeness properties of $\topkdp$
(Lemma~\ref{lemma:topk-closeness}) and our choice $\mstopkscale
= \frac{k a}{n\alpha_2}$
then give $\wt{D} \close{\alpha_2}{0} \wt{D}'$.

\section{Discussion}

The simplicity of mean estimation in classical statistics belies the
sophistication necessary to adaptively and accurately estimate a mean under
differential privacy constraints.  While we have developed (nearly) minimax
optimal procedures for mean estimation, a number of questions remain open,
and we hope that we or others will tackle them. From a practical
perspective, while our procedure is implementable, the numerical constant
factors we have maintained to guarantee privacy---in addition to the
logarithmic factors in $n$ and $\log\frac{1}{\delta}$---may make effective
use of the procedure challeenging. From a theoretical perspective, it is
still interesting to attempt to remove the logarithmic factors present in
our bounds. Additionally, while we can adapt to weaker than sub-Gaussian
moment bounds (via the method $\adamean$), it may be possible to provide a
sharper procedure or tighter analysis to achieve optimal dependence on
dimension $d$ and privacy level $\diffp$, as in the case that $p$ moments
are available, our results appear to be roughly a factor of
$(\sqrt{d}/\diffp)^{1/p}$ loose (recall
Examples~\ref{example:general-moment} and~\ref{example:general-moment-2}).
It will be interesting to see the extent of possibilities for differentially
private estimation in these more general cases.

\appendix

\section{Proofs of standard privacy results}

\subsection{Proof of Lemma~\ref{lemma:normal-closeness-mean}}
\label{sec:proof-normal-closeness-mean}

The first case follows from \citet[Theorem 3.22]{DworkRo14}.  For the
second, we use Mironov's R\'{e}nyi-differential privacy~\cite{Mironov17}.
The R\'{e}nyi $\alpha$-divergence between distributions
$P$ and $Q$ is
$\rendiv{P}{Q} = \frac{1}{\alpha - 1} \log \int (\frac{dP}{dQ})^\alpha
dQ$, and~\cite[Proposition 3]{Mironov17}
if $\rendiv{P}{Q} \le c$, then for
all measurable $A$ and $\delta  >0$ we have
$P(A) \le \exp(c + \frac{\log(1/\delta)}{\alpha - 1})
Q(A) + \delta$.
The R\'{e}nyi divergence for Gaussians has the explicit form
\begin{equation*}
  \rendiv{\normal(\mu_1, \tau^2  \Sigma)}{
    \normal(\mu_2, \tau^2 \Sigma)}
  = \frac{\alpha}{2 \tau^2} \norm{\mu_1 - \mu_2}_\Sigma^2.
\end{equation*}
When $\rho \ge \norm{\mu_1 - \mu_2}_\Sigma$, we set
$\beta = \alpha - 1$ and see that for
$\diffp$ satisfying
\begin{equation*}
  \diffp = \frac{\rho^2}{2 \tau^2}
  + \frac{\beta \rho^2}{2 \tau^2}
  + \frac{\log(1/\delta)}{\beta}
\end{equation*}
we obtain $\normal(\mu_1,
\tau^2 \Sigma) \closeed \normal(\mu_2, \tau^2 \Sigma)$.
Choosing $\beta$ to minimize the preceding
$\diffp$ we obtain
$\diffp = \frac{\rho^2}{2 \tau^2}
+ \frac{\rho}{\tau} \sqrt{2 \log \frac{1}{\delta}}$,
and solving for $\eta = \frac{1}{\tau}$ in
$\frac{\rho^2}{2} \eta^2 + \sqrt{2 \log\frac{1}{\delta}}
\rho \eta - \diffp$ yields
\begin{equation*}
  \tau = \frac{1}{\eta} = \frac{\rho}{\sqrt{2 \log \frac{1}{\delta}
      + 2 \diffp} - \sqrt{2 \log \frac{1}{\delta}}}
\end{equation*}
is always sufficient to guarantee
$\normal(\mu_1,
\tau^2 \Sigma) \closeed \normal(\mu_2, \tau^2 \Sigma)$.

\subsection{Proof of Lemma~\ref{lemma:normal-closeness-cov}}
\label{sec:proof-normal-closeness-cov}

Without loss of generality, we may assume $\mu = 0$.
We first reduce to the case where $\Sigma_1$ and $\Sigma_2$ are
full-rank.
Because $\dpsd(\Sigma_1, \Sigma_2) < \infty$, we have immediately
that there exists a vector space $V \subset \R^d$ with
$V = \colspace(\Sigma_1) = \colspace(\Sigma_2)$.
Letting $k = \dim(\colspace(\Sigma_1))$, take $U \in \R^{d\times k}$
to be an orthonormal matrix such that $V = \colspace(U)$.
The random variables $X \sim \normal\left(0, \Sigma_1\right)$ and
$Y \sim \normal\left(0, \Sigma_2\right)$ have
support $V$ and multiplication by $U^T$ is an isomorphism between $V$ and
$\R^k$, so $X \closeed Y$ if and only if $U^TX \closeed U^TY$.
Of course, $U^TX \sim \normal\left(0, U^T\Sigma_1 U\right)$ and
$U^TY \sim \normal\left(0, U^T\Sigma_2 U\right)$ and both $U^T\Sigma_1 U$
and $U^T\Sigma_2 U$ are full rank. The orthogonality of $U$ gives
$\dpsd(U^T\Sigma_1 U, U^T\Sigma_2 U) = \dpsd(\Sigma_1, \Sigma_2) \leq \gamma$.
Hence, by showing the lemma for the full-rank matrices $U^T\Sigma_1 U$ and
$U^T\Sigma_2 U$, we will have shown the claim for $\Sigma_1$ and $\Sigma_2$.

We proceed with the full-rank case with an argument similar
to \citet[Lemma 4.15]{BrownGaSmUlZa21}.
Define $D_1 = \Sigma_1^{1/2} \Sigma_2^{-1} \Sigma_1^{1/2} - I$ and
$D_2 = \Sigma_2^{1/2} \Sigma_1^{-1} \Sigma_2^{1/2} - I$.
As $D_1$ has the same spectrum as
$\Sigma_2^{-1/2} \Sigma_1 \Sigma_2^{-1/2} - I$, we have by assumption that
$\nucnorm{D_1} \leq \gamma$ and $\nucnorm{D_2} \leq \gamma$.

Let $f_1$ be the density of $P_1 = \normal\left(0, \Sigma_1\right)$ and
$f_2$ that of $P_2 = \normal\left(0, \Sigma_2\right)$.  Then, to show
$(\eps, \delta)$-closeness, it suffices to show $\ell(W) \defeq \la \log
\frac{f_1(W)}{f_2(W)} \ra \leq \eps$ with probability at least $1-\delta$
when $W$ is drawn from either $P_1$ or $P_2$.  By symmetry, it suffices to
only show this bound for the case when $W \sim P_1$.
Expanding $\ell$, we have
\begin{equation}
  \label{eqn:abs-log-ratio-gaussian}
  \ell(w)
  = \la \half \log \frac{\det(\Sigma_2)}{\det(\Sigma_1)}
  + \frac{1}{2}w^T(\Sigma_2^{-1} - \Sigma_1^{-1}) w \ra
  \leq \frac{1}{2} \la w^T(\Sigma_2^{-1} - \Sigma_1^{-1})w \ra
  + \frac{1}{2} \la \log \frac{\det(\Sigma_2)}{\det(\Sigma_1)} \ra.
\end{equation}
The final term is independent of $w$ and has the bound
\begin{align*}
  \la \log \frac{\det(\Sigma_2)}{\det(\Sigma_1)} \ra
  & = \max\left\{\log \det(\Sigma_2^{1/2} \Sigma_1^{-1} \Sigma_2^{1/2}), 
  \log \det(\Sigma_1^{1/2} \Sigma_2^{-1} \Sigma_1^{1/2}) \right\} \\
  & \leq \max\{\tr(D_2), \tr(D_1)\}
  \le \max\{\nucnorm{D_2}, \nucnorm{D_1}\} 
  \leq \gamma,
\end{align*}
where the first inequality holds because $\log \det(A) \leq
\tr(A-I)$ for any positive definite $A$.

Now we bound the first term on the right hand side of
inequality~\eqref{eqn:abs-log-ratio-gaussian} with high probability.
Since $W \sim P_1$, the whitened random variable $Z =
\Sigma_1^{-1/2}W \sim \normal(0, I)$.  We then
have
\begin{align*}
  \la W^T(\Sigma_2^{-1} - \Sigma_1^{-1})W \ra & = \la Z^T D_1 Z \ra,
\end{align*}
and so by the Hanson-Wright
inequality~\cite[e.g.][Thm.~6.2.1]{Vershynin19}, we have with probability
at least $1-\delta$ that
\begin{align*}
  \la Z^T D_1 Z \ra
  & \leq |\tr(D_1)| + 2\lfro{D_1}\sqrt{\log\frac{2}{\delta}}
  + 2\opnorm{D_1} \log\frac{2}{\delta} \leq 5\gamma \log\frac{2}{\delta},
\end{align*}
where the inequality holds because $\opnorm{D_1}
\le \lfro{D_1} \le \nucnorm{D_1} \le \gamma$ and
$\log \frac{2}{\delta} \ge 1$. Thus $\ell(W) \leq 6\gamma
\log\frac{2}{\delta} \leq \eps$ with probability at least $1-\delta$.

\bibliography{bib}
\bibliographystyle{abbrvnat}


\end{document}